\documentclass[twoside]{article}
\makeatother  
\usepackage{babel} 
\usepackage{times}
\usepackage{helvet}
\usepackage{courier}
\usepackage{multicol}
\usepackage{soul}
\usepackage{xparse}
\usepackage{amsmath}
\usepackage{amsthm}
\usepackage{amssymb}
\usepackage{xspace}
\usepackage{relsize}
\usepackage{graphicx}
\usepackage{multirow}
\usepackage{comment}
\usepackage{mathtools}
\usepackage{dsfont}
\usepackage{enumerate}
\usepackage{enumitem}
\usepackage{mathabx}
\usepackage{mathrsfs}
\usepackage{footnote}
\usepackage{multirow}
\usepackage{xcolor}
\usepackage{hyperref} 
\usepackage{mleftright} 
\mleftright

\usepackage[round]{natbib}
\renewcommand{\cite}{\citep}

\newcommand{\EE}{\mathbb{E}}

\newcommand{\Var}{\mathrm{Var}}
\newcommand{\imax}{i_{\max}}

\theoremstyle{definition}
\newtheorem{definition}{Definition}
\newtheorem{theorem}[definition]{Theorem}
\newtheorem{lemma}[definition]{Lemma}
\newtheorem{corollary}[definition]{Corollary}
\newtheorem{remark}[definition]{Remark}

\def\eps{{\epsilon}}

\usepackage[margin=1in]{geometry}

\begin{document}

\title{Order-Optimal Sequential 1-Bit Mean \\ Estimation in General Tail Regimes\footnotetext{A preliminary version of this work was presented at the 29th International Conference on Artificial Intelligence and Statistics (AISTATS 2026). }}


\author{Ivan Lau and Jonathan Scarlett}
\date{National University of Singapore}
\maketitle

\begin{abstract}
In this paper, we study the problem of mean estimation under 1-bit communication constraints. We propose a novel adaptive mean estimator based solely on randomized threshold queries, where each 1-bit outcome indicates whether a given sample exceeds a sequentially chosen threshold. Our estimator is $(\epsilon, \delta)$-PAC for any distribution with a bounded mean $\mu \in [-\lambda, \lambda]$ and a bounded $k$-th central moment $\mathbb{E}[|X-\mu|^k] \le \sigma^k$ for any fixed $k > 1$. Moreover, our sample complexity is order-optimal in all such tail regimes, i.e., for every such $k$ value. For $k \neq 2$, our estimator's sample complexity matches the unquantized minimax lower bounds plus an unavoidable $O(\log(\lambda/\sigma))$ localization cost. For the finite-variance case ($k=2$), our estimator's sample complexity has an extra multiplicative $O(\log(\sigma/\epsilon))$ penalty, and we establish a novel information-theoretic lower bound showing that this penalty is a fundamental limit of 1-bit quantization. We also establish a significant adaptivity gap: for both threshold queries and more general interval queries, the sample complexity of any non-adaptive estimator must scale linearly with the search space parameter $\lambda/\sigma$, rendering it vastly less sample efficient than our adaptive approach. 
Finally, we present algorithmic variants that (i) handle an unknown sampling budget, (ii) adapt to an unknown scale parameter~$\sigma$ given (possibly loose) bounds, (iii) require only two stages of adaptivity to achieve order-optimal sample complexity at the expense of more general 1-bit queries, and (iv) leverage multiple local samples per 1-bit query to proportionally reduce communication costs.
\end{abstract}

\section{Introduction}

Mean estimation is one of the most fundamental and ubiquitous tasks in statistics, machine learning, and theoretical computer science. In modern applications, such as those arising in large-scale sensor networks and decentralized federated learning, the learner often cannot directly access the raw data, and communication bottlenecks often mandate that data samples be severely compressed prior to transmission. We address an extreme case of this communication-constrained setting, where the learner receives only a single bit of feedback per sample. This extreme quantization raises a fundamental theoretical question: 
\begin{quote}
  \textit{How does 1-bit quantization affect the sample complexity of
  mean estimation?}
\end{quote}

We specifically focus on the threshold query model, where the learner sequentially sends a scalar threshold to an agent and receives a 1-bit indicator of whether the observed sample exceeds it.  Beyond its simplicity, the threshold query model naturally captures interesting real-world scenarios where observing the exact value of a sample is impossible, but binary threshold crossings are easily observed. A canonical example is pricing in economics~\cite{kleinberg2003value, leme2023pricing}: A seller cannot directly observe a buyer's maximum willingness-to-pay (their hidden sample), but by offering a price, the seller observes a 1-bit purchasing decision indicating whether the buyer's internal valuation exceeds said price. Similar mechanisms appear in bio-assay testing, where a specimen reacts only if a viral load exceeds a dosage threshold, and in reliability engineering, where a component fails if a stressor exceeds its physical limit.

A significant challenge in 1-bit mean estimation is the loss of spatial information. When the location of the distribution's core mass is highly uncertain (e.g., the mean lies somewhere in $[-\lambda, \lambda]$ for some $\lambda$ much higher than the variance), taking threshold queries in the ``wrong'' region may yield sequences of all zeros or all ones, providing virtually no statistical information. This problem is severely exacerbated when the underlying distribution exhibits heavy tails, as the estimator must distinguish between rare large outlier samples and the true center of mass without being able to observe the magnitude of the outliers.

In our preliminary conference version \cite{lau2025sequential}, we proposed an adaptive 1-bit mean estimator that achieves near-optimal sample complexity for distributions with a bounded $k$-th central moment for $k \ge 2$ (e.g., distributions with finite variance or sub-Gaussian tails). However, that preliminary framework suffered from several notable limitations. First, it was entirely unclear whether the framework could be generalized to handle heavy-tailed distributions where $k \in (1, 2)$. Second, it suffered from suboptimal logarithmic gaps between the upper and lower bounds. Finally, the estimator relied on the more demanding interval-query model, requiring the learner to effectively query two boundaries simultaneously.

In this paper, we address these issues in detail by restructuring and refining the framework, in particular attaining the following advantages:

\begin{itemize}

    \item \textbf{Threshold Queries and Heavy Tails:} We replace the interval-query model with the simpler and more practically relevant threshold query model. Furthermore, by generalizing the framework of our preliminary version, we extend our estimator to successfully handle heavy-tailed distributions where $k \in (1, 2)$.
    
    \item \textbf{Order-Optimality for all $k > 1$:} To estimate the mean using 1-bit feedback, our approach partitions the search space into regions to estimate ``local'' probability masses. We replace the previous $k$-dependent spatial partitioning scheme of the prior work with a simpler universal geometric grid. By pairing this single grid with a carefully tuned $k$-dependent sample allocation strategy, we eliminate the suboptimal logarithmic factors present in the conference version, achieving order-optimal sample complexity across all tail regimes $k > 1$.  While this is shown using a matching lower bound from the \emph{unquantized} setting when $k \ne 2$, for the finite-variance case $k = 2$ we further provide a novel lower bound under 1-bit quantization (not present in the conference version) that shows a multiplicative $\log(\sigma/\eps)$ factor to be unavoidable.  

    \item \textbf{Other Refinements and Extensions:} We refine the two-stage version of the estimator from the preliminary version by replacing its first stage with a global coding-theoretic procedure, yielding an order-optimal 1-bit mean estimator with only two stages of adaptivity. Moreover, we show that when each 1-bit query may depend on $m$ local samples, the communication cost can be reduced by roughly a factor of $m$.
    
\end{itemize}
See Section \ref{sec: contributions} for a more detailed summary of our contributions.

\subsection{Problem Setup}\label{sec: problem setup}
\textbf{Distributional assumptions.}
Let $X$ be a real-valued random variable\footnote{Our results also
  have implications for certain multivariate settings; see
Section~\ref{sec:multivar} for details.} with unknown distribution $D$.
We assume that $D$ belongs to a (non-parametric) family $\mathcal{D}
=  \mathcal{D}(k, \lambda, \sigma)$, defined by known parameters $k >
1$ and $\lambda > \sigma > 0$; a distribution $D$ is in this family
if the following conditions hold:
\begin{enumerate}[topsep=0pt, itemsep=0pt]
  \item Bounded mean: $\mu(D)  \in [-\lambda,
    \lambda]$,\footnote{Without loss of generality, we set the search
      range to be symmetric.  Note that a dependence on the search
      width $\lambda$ is unavoidable in the 1-bit setting (see
      Theorem~\ref{thm: adaptive lower bound}), but a crude upper bound
      can be used due to the mild logarithmic dependence in the sample
    complexity (see Theorem~\ref{thm: main}).}
  \item Bounded $k$-th central moment: $\EE|X-\mu|^k \le \sigma^k$,
\end{enumerate}
where $k$, $\lambda$, and $\sigma$ are known to the learner.  Note
that the support of $D$ may be unbounded.

\textbf{1-bit communication protocol.}
The learner is interested in estimating the population mean $\mu =
\mu(D) = \EE[X]$ from~$n$ independent and identically distributed
(i.i.d.) samples $X_1, \dotsc, X_n \sim D$, subject to a 1-bit
communication constraint per sample.
The estimation proceeds through an interactive protocol between a
learner and a single memoryless agent\footnote{Equivalently, this can
  be viewed as a sequence of memoryless agents where the agent in each
  round may be different.  In particular, the agent in round $t$ only
  has access to $X_t$ and not to the previous samples
$X_1,\dotsc,X_{t-1}$.} that observes i.i.d. samples and sends 1-bit
feedback to the learner.
Specifically, for $t = 1, \dotsc, n$:
\begin{enumerate}[topsep=0pt, itemsep=0pt]
  \item The learner sends a 1-bit quantization function $Q_t\colon
    \mathbb{R} \to \{0, 1\}$ to an agent;

  \item The agent observes a fresh sample $X_t \sim D$ and sends a
    1-bit message $Y_t = Q_t(X_t)$ to the learner.
\end{enumerate}
After $n$ rounds, the learner forms an estimate $\hat{\mu}$ based on
the entire interaction history $\big(Q_1, Y_1, \dotsc, Q_n, Y_n
\big)$. This setting (or similar) was also adopted in previous works on sequential
communication-constrained learning,
e.g.,~\cite{hanna2022solving, mayekar2023communication, lau2025quantile}.

The learner's algorithm in this protocol is formally defined as follows:
\begin{definition}[1-bit mean estimator]
  A 1-bit mean estimator is an algorithm for the learner that
  operates within the above communication protocol. It consists of
  \begin{enumerate}[topsep=0pt, itemsep=0pt]
    \item A (potentially randomized) query strategy for selecting the
      quantization functions $Q_1, \dotsc, Q_n$, where the choice of
      $Q_t$ can depend adaptively on the history of interactions
      $(Q_1, Y_1, \dotsc, Q_{t-1}, Y_{t-1})$.

    \item An estimation rule that maps the full transcript $(Q_1,
      Y_1, \ldots, Q_n, Y_n)$ to a final estimate $\hat{\mu} \in \mathbb{R}$.
  \end{enumerate}
  We say that an estimator is \emph{non-adaptive} if the query
  strategy selects all quantization functions in advance, without
  access to any of the outcomes $Y_1, \dotsc, Y_n$.
\end{definition}

\textbf{Threshold query model.} In the general problem
formulation, we place no restriction on the choice of quantization
function~$Q_t$.  However, motivated by the desire for ``simple''
choices in practice, we focus primarily on \emph{threshold queries}.  A threshold query yields a 1-bit indicator specifying whether a sample falls on a designated side of a spatial boundary. Formally, we define a threshold query as any quantization function of the form $Q_t(x) = \mathbf{1}\{x \le \gamma_t\}$ or $Q_t(x) = \mathbf{1}\{x \ge \gamma_t\}$ for some sequentially chosen threshold $\gamma_t \in \mathbb{R}$.\footnote{Since we do not assume the underlying distribution is continuous, the complement of the event $\{X_t \le \gamma\}$ is the strict inequality $\{X_t > \gamma\}$, which differs from $\{X_t \ge \gamma\}$ if the distribution contains a point mass at the threshold $\gamma$. For analytical convenience, we formally allow both inclusive inequalities ($\le$ and $\ge$) in our threshold query model.  However, even if only one direction is allowed, we can easily handle this by adding very slight (continuous-valued) randomization to the values of $a_i$ and $b_i$ used in our algorithm (see Section \ref{sec: algorithm}). 
} 
Our main estimator will only use such queries, though we will also
present a variant in Section~\ref{sec: two-stage} that utilizes general 1-bit quantization functions.

\textbf{Learner's goal.}
The learner's goal is to design a 1-bit mean estimator that returns
an accurate estimate with high probability, while using as few
samples as possible. We formalize this notion as follows:

\begin{definition}[$(\eps, \delta)$-PAC]
  A mean estimator is $(\eps, \delta)$-PAC for distribution
  family $\mathcal{D}$ with sample complexity~$n(\eps,\delta)$
  if, for each distribution $D \in  \mathcal{D}$, it returns an
  $\eps$-correct estimate $\hat{\mu}$ with probability at least
  $1-\delta$, i.e.,
  \begin{equation*}
    \text{for each } D \in \mathcal{D}, \quad
    \Pr\left(  |\hat{\mu} - \mu(D)| \le \eps \right) \ge 1- \delta
  \end{equation*}
  and the number of samples required is bounded by
  $n(\eps,\delta)$. The probability is taken over the samples
  $X_1, \ldots, X_n$ and any internal randomness of the estimator.
\end{definition}

\textbf{Notation.} We use standard asymptotic notation $O(\cdot)$, $\Omega(\cdot)$, and $\Theta(\cdot)$ to hide absolute constants. When these hidden constant factors depend on the moment parameter $k$, we make this dependence explicit using the subscripted notation $O_k(\cdot)$, $\Omega_k(\cdot)$, and $\Theta_k(\cdot)$.  Throughout the paper, the function $\log(\cdot)$ has base $e$.

\subsection{Summary of Contributions}\label{sec: contributions}
With the problem setup now in place, we summarize our main
contributions as follows:

\begin{itemize}[topsep=0pt, itemsep=0pt]
  \item \textbf{Adaptive 1-Bit Mean Estimator:}
  We propose a novel adaptive 1-bit mean estimator (see Section~\ref{sec: algorithm}) that relies solely on (randomized) threshold queries. It operates by first localizing the distribution's core via a noisy binary search, and subsequently estimating local probability masses over a universal geometric grid paired with a local variance-aware sample allocation strategy.

    \item \textbf{Order-Optimal Sample Complexity:}
    We prove that this estimator is $(\eps, \delta)$-PAC for the generalized distribution family $\mathcal{D}(k, \lambda, \sigma)$ for any fixed $k > 1$. Despite its structural simplicity, our approach strictly tightens and generalizes the bounds from our preliminary conference version~\cite{lau2025sequential}, entirely eliminating the suboptimal logarithmic factors caused by the different search space partitioning and some looseness in its analysis (see Remark~\ref{rem: unified framework}). Our resulting sample complexity is order-optimal across all tail regimes $k > 1$.
    
    \item \textbf{Lower Bound in the Finite Variance Case:}
    For $k \ne 2$, the sample complexity matches the \textit{unquantized} minimax rate plus an additive $\log(\lambda/\sigma)$ localization cost that we prove to be unavoidable (see Theorems~\ref{thm: main} and~\ref{thm: adaptive lower bound}). For the finite-variance case ($k=2$), our upper bound contains an additional $O(\log(\sigma/\eps))$ factor compared to the unquantized minimax rate. We establish a novel information-theoretic lower bound proving that this logarithmic penalty is an unavoidable consequence of 1-bit quantization, thereby confirming our estimator's optimality for $k=2$.
    
    \item \textbf{Lower Bound Proving an Adaptivity Gap:} Our adaptive sample complexity bound scales only logarithmically with $\lambda/\sigma$, which contrasts with existing bounds for communication-constrained non-parametric mean estimators scaling at least linearly in~$\lambda$ (see Section~\ref{sec: related work}). For the threshold-query and a more general interval-query model, we establish an ``adaptivity gap'' by showing a worst-case lower bound $\Omega(\lambda \sigma/\eps^2 \cdot \log(\delta^{-1}))$ for non-adaptive estimators (see Theorem~\ref{thm: non-adaptive lower bound}), in particular having a linear dependence on $\lambda$.

    \item \textbf{Algorithmic Extensions and Variations:}
    We extend our framework to accommodate several practical constraints (see Section~\ref{sec: variations and refinements}). We provide variants that adapt to an unknown sampling budget (yielding an anytime-valid estimator) and an unknown scale parameter $\sigma$ (see Sections~\ref{sec: unknown eps} and~\ref{sec: partially unknown scale}). We also demonstrate that order-optimal sample complexity can be achieved in just two stages of adaptivity by using general 1-bit queries and a global codebook for localization (see Section~\ref{sec: two-stage}). Finally, we establish that when each 1-bit query may depend on a local batch of $m$ samples, local averaging proportionally reduces the refinement communication cost across all tail regimes. (see Section~\ref{sec:m_samples_per_query}).
\end{itemize}

\subsection{Related Work}\label{sec: related work}
The related work on communication-constrained mean estimation
is extensive; we only provide a brief outline here, emphasizing the
most closely related works.

\textbf{Classical mean estimation.}
Mean estimation (in the unquantized setting) is a fundamental and well-studied problem in statistics, e.g., see~\cite{lee2022optimal, cherapanamjeri2022optimal, minsker2023efficient, dang2023optimality, gupta2024beyond} and the references therein. The state-of-the-art $(\eps, \delta)$-PAC estimator by~\cite{lee2022optimal} achieves a tight sample complexity $n = \left(2+o(1) \right) \cdot (\sigma^2/\eps^2) \cdot \log(1/\delta)$ for all distributions with finite variance~$\sigma^2$. Beyond the finite-variance regime, significant attention has been devoted to robust estimation for heavy-tailed distributions where only a fractional central moment $k \in (1, 2)$ is bounded \cite{bubeck2013bandits, devroye2016sub, lugosi2019mean}. In this regime, the unquantized minimax sample complexity is known to scale as $\Theta\big( (\sigma/\eps)^{\frac{k}{k-1}} \log(1/\delta) \big)$. Collectively, these unquantized rates serve as a natural benchmark for mean estimation problems under communication constraints.

\textbf{Communication-constrained estimation and learning}.
Early work in communication-constrained estimation, learning, and optimization was
motivated by the applications of wireless sensor networks
(see~\cite{xiao2006distributed, varshney2012distributed,
  veeravalli2012distributed, he2020distributed} and the references
therein), with a recent resurgence driven by the rise of large-scale
machine learning systems. This has led to the characterization of the
sample complexity or minimax risk/error for various communication-constrained
estimation problems~\citep{ZhangDJW13, GargMN14, Shamir14,
  BravermanGMNW16, XuR18, HanMOW18, HanOW18, BarnesHO19, BarnesHO20,
  AcharyaCT20a, AcharyaCT20b, acharya2021distributed,
  acharya2021interactive, acharya2021estimating, acharya2023unified,
shah2025generalized}.

While abundant, most of the existing literature differs in major aspects including the estimation goal itself, the use of parametric models, and/or imposing significantly stronger assumptions.
To our knowledge, none of the existing work on non-parametric communication-constrained estimation captures our problem setup. For example, non-parametric density estimation~\cite{BarnesHO20, AcharyaCST21a} is an inherently harder
problem, and accordingly the authors impose certain regularity conditions on the density function (e.g., belonging to Sobolev space).  Similarly, communication-constrained non-parametric function estimation
problems in~\cite{zhu2018distributed, Szab2018AdaptiveDM,
  Szab2020DistributedFE, cai2022distributednonparametric,
zaman2022distributed} assume certain tail bounds on the likelihood
ratio (e.g., Gaussian white noise model).

\textbf{Communication-constrained mean estimation.}
Several works study variants of mean estimation under communication constraints directly. A large body of work focuses on parametric settings, often assuming a known location-scale family~\cite{kipnis2022mean, kumar2025one} with a particular emphasis on Gaussians~\cite{ribeiro2006bandwidth1, cai2022distributed, cai2024distributed}.  These estimators crucially rely on CDF
inversion, which is highly dependent on exact knowledge of the
parametric family, and thus unsuitable for our non-parametric setting.
The non-parametric mean estimators in~\cite{luo2005universal,
ribeiro2006bandwidth2} can handle broader distributional families, but
require additional assumptions such as bounded support and/or smooth
density functions. Furthermore, some of these estimators require more
than 1 bit of feedback (per coordinate) per sample.
A recent independent work on non-adaptive 1-bit mean
  estimation \cite{abdalla2026robust} partially circumvents these
  restrictive assumptions.  However, their estimator adopts a fixed
  quantization range whose width scales as
  $\Omega(\sigma^2/\eps)$ in the
  worst case,\footnote{To bound the worst-case
      truncation bias by $O(\eps)$ under only a finite-variance
      assumption, it can be shown that one must set the range to be
      $\Omega(\sigma^2/\eps)$ due to the worst-case tightness of Cantelli's inequality (a one-sided version of Chebyshev's inequality).
  } and this translates to a sample complexity of $\Omega(\sigma^4 /
  \eps^4)$.
  In contrast, our adaptive 1-bit mean estimator achieves
  $\widetilde{O}(\sigma^2/\eps^2)$ rates for all distributions
whose first two moments lie within known bounds.

\textbf{Empirical vs. population mean estimation.}
A closely related line of work focuses on distributed
\textit{empirical} mean estimation of a fixed dataset, which is a key
primitive in federated learning~\cite{suresh2017distributed,
  konevcny2018randomized, davies2020new, vargaftik2021drive,
  mayekar2021wyner, vargaftik2022EDEN, benbasat2024accelerating,
babu2025unbiased}.
These estimators typically achieve a minimax optimal mean squared
error (MSE) that scale as $\EE[ { (\hat{\mu} - \mu_{\text{emp}})
}^{2} ] = O(\lambda^2/n)$.
By using Markov's inequality and the median-of-means method, this can
be converted to $(\eps, \delta)$-PAC \textit{population} mean
estimation with a sample complexity of $n =
\widetilde{O}(\lambda^2/\eps^2 \cdot \log(1/\delta))$.
In contrast, our mean estimator achieves a sample complexity of
$\widetilde{O}(\sigma^2/\eps^2 \cdot \log(1/\delta) +
\log(\lambda/\sigma))$, which is considerably smaller when $\sigma^2
\ll \lambda^2$.
Although some empirical mean estimators achieve an MSE that depends on the 
empirical deviation/variance $\sigma_{\text{emp}}$ of the fixed
dataset~\cite{ribeiro2006bandwidth2, suresh2022correlated}, they
require a bounded support.
Furthermore, their MSE scales at least linearly with $\lambda$, e.g.,
the one in~\cite{suresh2022correlated} scales as $\EE[ {(\hat{\mu} -
\mu_{\text{emp}})}^2 ] = O(\sigma_{\text{emp}} \lambda/n + \lambda^2/n^2)$.
Consequently, converting them to $(\eps, \delta)$-PAC
\textit{population} mean estimator using standard techniques would
result in a sample complexity bound that scales at least linearly
with~$\lambda$.

\section{Estimator and Upper Bound}
In this section, we introduce our 1-bit mean estimator and provide
its performance guarantee.  Our estimator is designed as a
target-accuracy driven procedure that takes parameters
$(k, \lambda,\sigma,\eps,\delta)$ as input.  It operates to ensure
the desired accuracy $\eps$ is attained with probability at least
$1-\delta$ while minimizing the sample complexity $n$, and hence does
not have an explicit pre-specified sample budget.  However, the
estimator can readily be applied to the fixed-budget setting where
the sampling budget is given and the goal is to minimize the
estimation error $\eps$. In Section~\ref{sec: unknown eps}, we
address a harder variant of this, where $n$ is fixed but unknown to the learner.

Before detailing the mechanics of the estimator, we first establish a structural property of the distribution family that simplifies the analysis for (very) light-tailed distributions.

\begin{remark}[Reduction to $k \le 3$]
\label{rem: operative moment}
    By Lyapunov's inequality, any distribution in $\mathcal{D}(k, \lambda, \sigma)$ (defined in Section \ref{sec: problem setup}) for $k > 3$ satisfies 
    \[
        \left(\EE[|X-\mu|^3]\right)^{1/3} \le \left(\EE[|X-\mu|^k]\right)^{1/k} \le \sigma.
    \]
    Thus, when $k > 3$, any distribution in $\mathcal{D}(k, \lambda, \sigma)$ is also a member of $\mathcal{D}(3, \lambda, \sigma)$. Consequently, for any distribution with ``very light'' tails (e.g., sub-Gaussian), our estimator can safely process the samples using an ``operative'' moment parameter of $3$. Therefore, without loss of generality, we assume the moment parameter satisfies $k \in (1, 3]$ for the rest of the paper.
    This will be convenient for the analysis in Appendix~\ref{appendix: proof of main result}, ensuring that $k$-dependent constants (e.g., $2^k$) remain suitably bounded.
\end{remark}

\subsection{Description of the Estimator}\label{sec: algorithm}
Our estimator first localizes an interval $I$ of length $O(\sigma)$
containing the mean $\mu$
with high probability (see Step~1 below).
Using the mid-point of $I$ as the ``centre", it identifies a cutoff
threshold $t$ such that
the contribution to the mean from the ``insignificant region'' $|X| \ge t$ is at most $\eps/2$ (see Step~2).
Finally, it forms an estimate of the mean contribution from the
``significant region'' $|X| < t$ to within an additive error of $\eps/2$ (see Steps~3--6).

This high-level strategy of performing ``localization'' (coarse estimation) and
``refinement'' (finer estimation) has appeared in prior works such
as~\cite{cai2022distributed},
but with very different details, particularly for refinement.
The key idea behind our refinement procedure is to partition the
significant region into sub-regions, and optimally allocate samples to
estimate the contribution from each sub-region using randomized threshold queries.

For the confidence parameter, we distinguish between the failure probability
allocated to localization, denoted by $\delta_{\rm loc}$, and the failure
probability allocated to refinement, denoted by $\delta_{\rm ref}$. These
parameters are chosen so that $\delta_{\rm loc} + \delta_{\rm ref} \le \delta$. In Theorem~\ref{thm: main}, we take $\delta_{\rm loc} = \delta_{\rm ref} = \delta/2$.

Our mean estimator is outlined as follows, with any omitted details deferred to Appendix~\ref{appendix: proof of main result}:
\begin{enumerate}
  \item \textbf{Localization}: If $\lambda\le 4\sigma$, we simply set $I=[-\lambda,\lambda]$, which contains $\mu$ deterministically and has length at most $8\sigma$. Otherwise, let $M \coloneqq \inf\{x\in\mathbb R : \Pr(X\le x)\ge 1/2\}$ be the (lower) median of $X$. Using existing median estimation techniques, we construct a random interval $[L,U]$ satisfying $U-L\le 6\sigma$ and $\Pr(M\in[L,U]) \ge 1-\delta_{\rm loc}$ using
\[
  n_{\rm loc}(\delta_{\rm loc},\lambda,\sigma)
  =
  \Theta\left(
    \log\frac{\lambda}{\sigma}
    +
    \log\frac{1}{\delta_{\rm loc}}
  \right)
\]
1-bit threshold queries. Conditioned on the event $\{M\in[L,U]\}$,
the property $|\EE[X]-M|\le \sigma$ implies that
\[
    \mu \in [L-\sigma, U+\sigma].
\]
Then the enlarged interval $I = [L-\sigma, U+\sigma]$ has length at most $8 \sigma$. Without loss of generality, we shift the coordinate system so that the midpoint of $I$ is $0$, i.e., the shifted mean is contained in $[-4\sigma, 4\sigma]$.

  \item \textbf{Cutoff Threshold Selection}: 
  Conditioned on the localization event from Step~1, we work in the recentered coordinate system, so that $|\mu|\le 4\sigma$.  We identify a preliminary cutoff level
    \[
       t_0
       =
       C_k \,\sigma \cdot \left(\frac{\sigma}{\eps}\right)^{1/(k-1)},
    \]
    where $C_k > 0$ is a sufficiently large constant depending only on $k$,
    chosen in particular so that $t_0 > 8\sigma$. For a distribution with
    bounded $k$-th central moment, this choice guarantees that any cutoff
    $s \ge t_0$ satisfies
    \[
       \EE \left[ |X| \cdot \mathbf{1}\{|X| \ge s\} \right] 
      \le \eps/2.
    \]
    The final cutoff $t$ will be chosen in Step~3 by rounding $t_0$ up based on some
    grid boundary. It will therefore satisfy $t\ge t_0$, and hence
    \[
       \left| \EE \left[X \cdot \mathbf{1}\{|X|\ge t\}\right] \right|
       \le \EE \left[ |X| \cdot \mathbf{1}\{|X|\ge t\}\right]
       \le \eps/2.
   \]
    It remains to form a final high-probability estimate $\hat{\mu}$ of the
    ``clipped mean'' $ \EE \left[ X \cdot \mathbf{1} \left(|X| <
      t  \right) \right]$ to within additive error $\eps/2$.

  \item \textbf{Significant Region Partitioning}: 
  We partition the significant region~$(-t, t)$ into symmetric regions
    $R_1, R_{-1}, R_2, R_{-2}, \ldots, R_{\imax}, R_{-\imax}$ 
    defined by exponentially growing interval boundaries $m_i \sigma$: 
    \begin{equation}\label{eq: general R_i def}
      R_i  =
      \begin{cases}
        \sigma \cdot \left[m_{i-1} , m_i  \right) & \text{if }  i \ge 1
        \\ \\
        -R_{|i|} & \text{if } i \le -1
      \end{cases}
      \quad \text{where} \quad
      m_0 = 0 \text{ and } 
      m_i = 2^{i} \text{ for } i \ge 1.
    \end{equation}
    We choose $\imax$ and $t$ to satisfy the requirement in Step 2:
    \[
        i_{\max} 
        \coloneqq \min\{i\ge1 : m_i\sigma \ge t_0\}
         = \left\lceil \log_2\left( \frac{t_0}{\sigma} \right) \right\rceil
        = \Theta\left( \frac{1}{k-1} \log\left(\frac{\sigma}{\eps}\right) \right)
        \quad \text{and} \quad
        t \coloneqq m_{i_{\max}}\sigma \ge t_0.
    \]
    The regions are disjoint except for the harmless overlap
    $R_1\cap R_{-1}=\{0\}$, whose contribution to the mean is zero.
    Hence the clipped mean decomposes as the sum of the local contributions $\mu_i \coloneqq \EE[X \cdot \mathbf{1}(X \in R_i)]$, and it suffices to estimate each $\mu_i$ separately:
    \[
      \sum_{1\le |i|\le i_{\max}}\mu_i
      = \sum_{1\le |i|\le i_{\max}} \EE[X \cdot \mathbf{1}\{X\in R_i\}]
      = \EE[X \cdot \mathbf{1}\{|X|<t\}].
    \]

  \item \textbf{Region-Wise Threshold Queries}: 
    We first describe the local mean estimation procedure for positive or right-sided regions $R_i=[a_i,b_i)$ where $i \ge 1$.
    We define auxiliary probabilities based on a random threshold $T_i \sim \mathrm{Unif}(a_i, b_i)$ as follows:
    \begin{equation}\label{eq: p_ai}
    p_{a_i} \coloneqq 
    \Pr(X \ge a_i) - \Pr(X \ge T_i) = 
    \EE[\mathbf{1}\{X \ge a_i\}] - \EE[\mathbf{1}\{X \ge T_i\}],
    \end{equation}
    and
    \begin{equation}\label{eq: p_bi}
    p_{b_i} \coloneqq 
    \Pr(X \ge T_i) - \Pr(X \ge b_i) = 
    \EE[\mathbf{1}\{X \ge T_i\}] - \EE[\mathbf{1}\{X \ge b_i\}].
    \end{equation}
    By exploiting the identity
    \begin{equation*}
      \mu_i = a_i \cdot p_{a_i} + b_i \cdot p_{b_i},
    \end{equation*}
    the task of estimating $\mu_i$ reduces to estimating the probabilities $p_{a_i}$ and $p_{b_i}$.
    As suggested by~\eqref{eq: p_ai} and~\eqref{eq: p_bi}, these can be estimated via threshold queries.
    Specifically, for a predetermined sample budget $n_i$ (specified in Step 5), the learner collects $n_i$ independent 1-bit responses for each of the following four threshold query types: 
    \begin{equation*}
      \mathbf{1}\{X \ge a_i\}, \quad \mathbf{1}\left\lbrace X \ge T_{i}^{(1)}\right\rbrace, \quad
      \mathbf{1}\left\{ X \ge T_i^{(2)} \right\}, \quad \text{and } \quad \mathbf{1}\{X \ge b_i \},
    \end{equation*}
    where the data samples $X$ and random thresholds $T_i^{(1)}, T_i^{(2)} \sim \mathrm{Unif}(a_i, b_i)$ are freshly sampled for each query. 
    Taking the empirical averages of these responses yields the unbiased probability estimates $\hat{p}_{a_i}$ and $\hat{p}_{b_i}$, which are combined to form the unbiased local estimate $\hat{\mu}_i = a_i \cdot \hat{p}_{a_i} + b_i \cdot \hat{p}_{b_i}$.

    For the negative or left-sided regions $R_{-i} = -R_i = (-b_i, -a_i]$ for $i \ge 1$, we apply the same procedure to the reflected variable $Z = -X$. Using $X \in R_{-i}$ if and only if $Z\in R_i = [a_i, b_i)$, we have
    \[
    \mu_{-i} = \EE[X \cdot \mathbf{1}\{X\in R_{-i}\}]
             = -\EE [Z \cdot \mathbf{1}\{Z\in R_i\}].
    \]
    Thus the same randomized threshold query procedure applied to $Z$ gives an unbiased estimate of $-\mu_{-i}$, and hence $\mu_{-i}$. The threshold queries involved are of the form $\mathbf{1} \{Z \ge \gamma \}$, or equivalently $\mathbf{1} \{X \le -\gamma \}$ in the original $X$-coordinate, which is allowed under our threshold-query model (see problem setup in~Section~\ref{sec: problem setup}).

  \item \textbf{Base Estimator and Sample Allocation}: 
  Summing the local estimates from all significant regions yields a single unbiased ``base estimator'' for the clipped mean:
  \[
    \hat{\mu}_{\rm base} = \sum_{1\le |i| \le i_{\max}} \hat{\mu}_i.
  \]
  To achieve the final $(\eps, \delta)$-PAC guarantee through median-of-means (see Step 6), it is sufficient for this base estimator to satisfy a global variance constraint of the form $\Var(\hat{\mu}_{\rm base}) = \sum_{i} \Var(\hat{\mu}_i) = O(\eps^2)$. 
  The learner achieves this by setting the regional sample budget $n_i$ to scale according to the decay of the local variance $\Var(\hat{\mu}_i)$. 
  Specifically, the sample allocation is set to:
  \begin{equation*}
    n_i = \Theta\left( \frac{\sigma^2}{\eps^2} \cdot 2^{|i|(2-k)}  \right).
  \end{equation*}
This allocation guarantees the target variance while yielding an order-optimal sample complexity for a single base estimator:
\begin{equation*}
    \sum_{|i| \le i_{\max}} n_i 
    =
    \begin{cases}
      O_k \left(  \dfrac{\sigma^2}{\eps^2} \right) & \text{if } k > 2
      \\ \\
      O \left(  \dfrac{\sigma^2}{\eps^2} \cdot
        \log \left( \dfrac{\sigma}{\eps}\right)  \right) &
      \text{if } k = 2
      \\ \\
      O_k \left(  \left( \dfrac{\sigma}{\eps}  \right)^{\frac{k}{k-1}}  
      \right)  & \text{if } k \in (1, 2),
    \end{cases}
  \end{equation*}
  where $O_k(\cdot)$ represents $O(\cdot)$ notation with a hidden constant that depends on $k$.

  \item \textbf{Median-of-Means}:
   While the base estimator $\hat{\mu}_{\rm base}$ successfully bounds the variance to $O(\eps^2)$, it provides an $\eps$-accurate estimate with only a constant probability. To boost the refinement success probability to $1 - \delta_{\rm ref}$, the learner wraps the base estimator in a median-of-means framework, 
   repeating Steps~4 and~5 independently $K = \Theta(\log(1/\delta_{\rm ref}))$ times to generate $K$ independent base estimates $\hat{\mu}_{\rm base}^{(1)}, \dots, \hat{\mu}_{\rm base}^{(K)}$. 
   The final output of the 1-bit mean estimator is their median:
   \begin{equation*}
    \hat{\mu} = \mathrm{median}\left( \hat{\mu}_{\rm base}^{(1)}, \dots, \hat{\mu}_{\rm base}^{(K)} \right).
   \end{equation*}

\end{enumerate}

\begin{remark}[Algorithmic Advancements over Preliminary Version]
\label{rem: algo advancements}
While the high-level spatial partitioning architecture shares structural similarities with our preliminary conference version \cite{lau2025sequential}, the framework of the estimator has been carefully refined to achieve order-optimality for all tail regimes. First, we replace the interval-query model with a simpler threshold-query model, estimating local probability masses purely through differences in empirical threshold crossings (Step 4). Second, we extend the estimation framework to handle heavy-tailed distributions where $k \in (1, 2)$ (Step 5), while simultaneously discarding the complicated $k$-dependent partition boundaries of the prior work in favor of a simpler universal geometric grid ($m_i = 2^i$ in Step 3).  Finally, instead of relying on crude union bounds to guarantee the accuracy of every local region simultaneously, we deploy a variance-aware local sample allocation (Step 5) paired with a median-of-means aggregator (Step 6). This effectively decouples the global $(\eps, \delta)$-PAC requirement from the granularity of the spatial grid, which is the key to eliminating the suboptimal logarithmic sample complexity penalties present in the previous framework (see Remark \ref{rem: unified framework} for details).
\end{remark}

\subsection{Upper bound}
We now formally state the main result of this paper, which is the
performance guarantee of our mean estimator in Section~\ref{sec:
algorithm}. The proof is deferred to Appendix~\ref{appendix: proof of
main result}, where we also provide the omitted details in the above outline.

\begin{theorem}\label{thm: main}
Fix $k > 1$, as well as $\lambda > \sigma > \eps > 0$, and $\delta \in (0, 1/2)$.\footnote{Throughout the upper bounds, logarithmic ratio terms are written
without constant floors for readability; formally, each $\log(a/b)$ term may
be read as $\max\{1,\log(a/b)\}$ when $a/b$ is not bounded below by a fixed
constant greater than one. Equivalently, one may run
the estimator with constant-factor inflated values of $\lambda$ or
$\sigma$ when needed; this changes the stated sample complexities by at
most constant factors.}
Run the mean estimator in Section~\ref{sec: algorithm} with
$\delta_{\rm loc} = \delta_{\rm ref} = \delta/2.$
Then, the estimator is $(\eps, \delta)$-PAC for distribution family $\mathcal{D}(k, \lambda,
\sigma)$ defined in Section \ref{sec: problem setup}, with sample complexity
\begin{align}
n
=
O\left(\log \left(\dfrac{\lambda}{\sigma} \right) \right) +
\begin{cases}
  O_k \left(  \dfrac{\sigma^2}{\eps^2}    \cdot
  \log\left(\dfrac{1}{\delta}\right) \right) & \text{if } k > 2
  \\ \\
  O \left(  \dfrac{\sigma^2}{\eps^2} \cdot
    \log \left( \dfrac{\sigma}{\eps}\right)  \cdot \log\left(
  \dfrac{1}{\delta}\right) \right) &
  \text{if } k = 2
  \\ \\
  O_k \left(  \left( \dfrac{\sigma}{\eps}  \right)^{\frac{k}{k-1}}  
  \cdot \log\left(\dfrac{1}{\delta}\right) \right)  & \text{if
  } k \in (1, 2),
\end{cases}
\label{eq: scaling law of n}
\end{align}
where $O_k(\cdot)$ represents $O(\cdot)$ notation with a hidden constant that depends on $k$. Note that the $O(\log(1/\delta_{\rm loc})) = O(\log(1/\delta))$ term from localization cost is absorbed into the refinement cost since $\sigma\ge \eps$.
\end{theorem}

Next, suppose that $X - \mu$ is sub-Gaussian with known parameter
$\sigma^2$. Then $X$ has its finite third central moment bounded by
$(C \sigma)^3$ for some absolute constant $C$. Therefore, the above
mean estimator can be used for sub-Gaussian distributions too.
\begin{corollary}
  \label{cor: subgaussian variant}
  Suppose that $X - \mu$ is sub-Gaussian with known parameter $\sigma^2$.
  Then there exists an $(\eps, \delta)$-PAC 1-bit mean estimator
  with sample complexity
  \begin{equation*}
    n = O \left(  \frac{\sigma^2}{\eps^2}  \cdot
    \log\left(\frac{1}{\delta}\right) + \log\frac{\lambda}{\sigma}\right).
  \end{equation*}
\end{corollary}

Our sample complexity matches the  minimax lower bound for the unquantized setting (see~\cite[p.2]{lee2022optimal} and~\cite[Theorem 3.1]{devroye2016sub}) up to
constant factors for $k \ne 2$ and up to a multiplicative $\log(\sigma/\eps)$ factor for $k = 2$,
along with an additional $\log(\lambda/\sigma)$ term for all $k > 1$; both of these extra terms are shown to be unavoidable in Theorem~\ref{thm: adaptive lower bound} below.
We also study variants where $(\eps,\sigma)$ are not prespecified
in Sections~\ref{sec: unknown eps} and~\ref{sec: partially unknown scale};
and a variant that uses only two rounds/stages of adaptivity in
Section~\ref{sec: two-stage}.

\begin{remark}[Tightened Rates and Order-Optimality]
  \label{rem: unified framework}
  The algorithmic advancements described in Remark~\ref{rem: algo advancements} yield strictly tightened and generalized upper bounds compared to our preliminary conference version~\cite{lau2025sequential}. Specifically, we achieve the following improvements across the tail regimes:
  \begin{enumerate}
    \item \textbf{Heavy-tailed distributions ($k \in (1, 2)$):} We achieve the order-optimal sample complexity for heavy-tailed distributions, an important regime that was entirely unaddressed in the preliminary version.  Moreover, we achieve this in a unified manner with the case $k \ge 2$, not requiring any structural changes.
        
    \item \textbf{Light-tailed and sub-Gaussian distributions ($k > 2$):} We completely eliminate the suboptimal doubly logarithmic factors for light-tailed distributions, and iterated logarithmic factors for sub-Gaussian distributions, yielding strict order-optimal rates.
    
     \item \textbf{Finite-variance distributions ($k=2$):} We shave two logarithmic factors off the previous sample complexity, matching the unquantized minimax bound up to a single $O\big(\log(\sigma/\eps)\big)$ factor (as opposed to the $O\big(\log^3(\sigma/\eps)\big)$ gap in \cite{lau2025sequential}). We establish in Section~\ref{sec: lower bounds} that this remaining logarithmic factor is not an artifact of our spatial partitioning, but rather a fundamental information-theoretic limit of 1-bit quantization. Consequently, our estimator achieves order-optimality across all tail regimes $k > 1$.
  \end{enumerate}
\end{remark}

\begin{remark}[$k$-Dependent Constant Factors and the Phase Transition at $k=2$]
\label{rem: k-dependent upper bound constants}
    The subscripted asymptotic notation $O_k(\cdot)$ hides $k$-dependent constant factors that diverge as the moment parameter~$k$ approaches the finite-variance boundary ($k \to 2$). Specifically, as we will discuss in the proof of Theorem \ref{thm: main}, the hidden $k$-dependence scales as $\Theta\big(\max(1, (k-2)^{-1})\big)$ for the light-tailed regime $k > 2$, and $\Theta\big(\max(1, (2-k)^{-1})\big)$ for the heavy-tailed regime $k \in (1, 2)$. This dual-sided divergence characterizes a phase transition in the estimator's behavior arising from the spatial partitioning architecture. As $k \to 2$, the geometric series governing our sample allocation flattens into a uniform sum (see Step 5(c) of Appendix~\ref{appendix: proof of main result}), manifesting the $O\big(\log(\sigma/\eps)\big)$ penalty for $k=2$ that is an inescapable information-theoretic cost of 1-bit quantization (see Section~\ref{sec: lower bounds}).
\end{remark}

\section{Lower Bounds and Adaptivity Gap}
\label{sec: lower bounds}

In this section, we further characterize the information-theoretic limits of 1-bit mean estimation via lower bounds on the sample complexity.  We first provide, in Theorem~\ref{thm: adaptive lower bound}, a minimax lower bound that matches our upper bound in Theorem~\ref{thm: main} up to constants across all tail regimes. In particular, we show that the additive $\log(\lambda/\sigma)$ localization term is unavoidable, and we reveal a novel quantization penalty unique to finite-variance distributions. 
Subsequently, in Theorem~\ref{thm: non-adaptive lower bound}, we show  that the best non-adaptive mean estimator is strictly worse than our
adaptive mean estimator when only threshold queries are allowed, and that the same holds even under a more general interval query model.  
This establishes a strict ``adaptivity gap'' between the performance of adaptive and non-adaptive query based mean estimators under threshold and/or interval queries. 
The proofs for all of our lower bounds are given in Appendix~\ref{appendix: Lower Bounds}.

\begin{theorem}[Matching Lower Bound]\label{thm: adaptive lower bound}
    Fix $k > 1$. There exist constants $c_k > 0$ and $\delta_0 \in (0, 1/2)$ such that, for all $\lambda \ge \sigma > 0$, all $\eps \in (0, c_k \sigma)$ and all $\delta \in (0, \delta_0)$, the following holds. For any $(\eps, \delta)$-PAC 1-bit mean estimator, there exists a distribution $D \in \mathcal{D}(k, \lambda, \sigma)$ such that the number of samples~$n$ must satisfy
  \[
    n = 
    \Omega\left(\log \left(\dfrac{\lambda}{\sigma} \right) \right) +
    \begin{cases}
      \Omega \left(  \dfrac{\sigma^2}{\eps^2} \cdot \log\left(\dfrac{1}{\delta}\right) \right) & \text{if } k > 2
      \\ \\
      \Omega \left(  \dfrac{\sigma^2}{\eps^2} \cdot \log\left(\dfrac{\sigma}{\eps}\right) \cdot \log\left(\dfrac{1}{\delta}\right) \right) & \text{if } k = 2
      \\ \\
      \Omega \left(  \left( \dfrac{\sigma}{\eps}  \right)^{\frac{k}{k-1}} \cdot \log \left(\dfrac{1}{\delta}\right) \right)  & \text{if } k \in (1, 2).
    \end{cases}
  \]
\end{theorem}

\begin{remark}[Finite-Variance Penalty]
\label{rem: separation of k=2 and k>2}
A surprising feature of Theorem~\ref{thm: adaptive lower bound} is the strict separation between finite-variance ($k=2$) and light-tailed ($k>2$) distributions. In the unquantized setting, both regimes share a base sample complexity of $\Theta(\sigma^2/\eps^2 \cdot \log(1/\delta))$, whereas Theorem~\ref{thm: adaptive lower bound} establishes that 1-bit communication constraints force the $k=2$ regime to uniquely incur an additional multiplicative $\log(\sigma/\eps)$ penalty. 
\end{remark}

To state the second result, we formally define an \emph{interval query} to be any query of the form $Q_t(x) = \mathbf{1}\{x \in [\alpha_t, \beta_t] \}$ for some pair $(\alpha_t, \beta_t)$, possibly with $\alpha_t = -\infty$ or $\beta_t = \infty$.  When $\alpha_t = -\infty$ (respectively, $\beta_t = \infty$), we trivially recover the threshold query $\mathbf{1}\{ x \le \beta_t \}$ (respectively, $\mathbf{1}\{ x \ge \alpha_t \}$), so interval queries strictly generalize threshold queries.

\begin{theorem}[Adaptivity Gap]\label{thm: non-adaptive lower bound}
  There exists constants $\delta_0 \in (0, 1/2)$ such that for all $\lambda \ge 2 \sigma > 0$, all $\eps \in (0, \sigma/2)$, and all $\delta \in (0, \delta_0)$, the following holds.
  For any non-adaptive $(\eps, \delta)$-PAC mean estimator utilizing only interval queries, there exists a distribution supported on an interval of length $\sigma$ (and therefore sub-Gaussian) with mean $\mu \in [-\lambda, \lambda]$ such that the total number of queries $n$ must satisfy:
    \[
        n = \Omega\left( \frac{\lambda \sigma}{\eps^2} \cdot \log \left(\frac{1}{\delta} \right) \right).
    \]
    Because the family of distributions with bounded support of length $\sigma$ is a strict subset of $\mathcal{D}(k, \lambda, \sigma)$ for all $k \ge 1$, this non-adaptive lower bound universally applies to all tail regimes studied in this paper.
\end{theorem}

In the remainder of this section, we discuss the high-level proof ideas.  
Setting aside the multiplicative $\log(\sigma/\eps)$ penalty unique to the $k=2$ regime momentarily, the remaining lower bounds (i.e., tail regimes for $k \ne 2$ and the additive $\log(\lambda/\sigma)$ localization cost) are established by constructing a finite ``hard subset'' of distributions that capture two distinct sources of difficulty: 
(i) ``coarsely'' identifying the distribution's location in $[-\lambda,\lambda]$ among $\Theta(\lambda/\sigma)$ possibilities,
and (ii) ``finely'' estimating the mean by distinguishing between two candidate distributions at that location whose means differ by $2\eps$. 
The fine estimation step inherently dictates the ``base'' sample complexity for each respective tail regime based on standard hypothesis testing lower bounds, i.e., $\Omega(\frac{\sigma^2}{\eps^2} \cdot \log(1/\delta))$ for $k \ge 2$ and $\Omega( (\frac{\sigma}{\eps})^{\frac{k}{k-1}} \cdot \log(1/\delta))$ for $k \in (1, 2)$. However, the dependency on $\lambda/\sigma$ arising from the coarse identification step differs fundamentally in adaptive vs. non-adaptive settings:
\begin{itemize}[topsep=0pt, itemsep=0pt]
  \item  In Theorem~\ref{thm: adaptive lower bound} (adaptive setting), 
  we can simply interpret the additive logarithmic dependence as the number of bits needed to identify the correct
  location among the $\Theta(\lambda/\sigma)$ possibilities, with each query giving at most 1 bit of information.

  \item  In Theorem~\ref{thm: non-adaptive lower bound} (non-adaptive setting), 
  the \emph{multiplicative} dependence arises because the estimator needs to allocate enough queries in \emph{every} one of
  the $\Theta(\lambda/\sigma)$ possible locations simultaneously, as it does not know the correct location in advance.
\end{itemize}

We note that the distributed Gaussian mean estimator in~\cite{cai2024distributed} is non-adaptive and achieves an order-optimal MSE. 
However, their estimator is highly specific to Gaussian distributions, and their quantization functions are based on Gray codes, exhibiting periodic behavior that is ``far'' from being a threshold or interval query.  

Next, we outline the proof idea for the lower bound showing that the $\log(\sigma/\eps)$ penalty is unavoidable when $k=2$.  
We define a hard-to-distinguish hypothesis testing problem over a geometric grid $x_j = 2^j \cdot \sigma$ for $j = 1, \dots, M = \Theta\big(\log(\sigma/\eps)\big)$.  The two distributions are as follows:
\begin{itemize}[topsep=0pt, itemsep=0pt]
\item The \textbf{null distribution $D_0$} is a symmetric, zero-mean distribution obtained by placing probability mass $q_j = \Theta\big(\frac{\sigma^2}{M x_j^2}\big)$ at each pair of points $\pm x_j$, with the remaining mass placed at the origin. Each pair contributes $\Theta(\sigma^2/M)$ to the variance, exhausting the prescribed variance budget $\sigma^2$.

\item The \textbf{alternative $\bar{D}$} is a uniform mixture $\bar{D} = \frac{1}{M}\sum_{j=1}^M D_j$, where each constituent distribution $D_j$ is formed from $D_0$ by shifting a small mass $p_j = \Theta(\eps/x_j)$ from $-x_j$ to $+x_j$. This creates a mean shift of $\Theta(\eps)$ while satisfying the global variance constraint.
\end{itemize}
Note that sampling from $\bar{D}$ is equivalent to first picking an index $j$ uniformly at random and then drawing a sample from $D_j$. Because the learner does not know which $j$ was chosen, any 1-bit query faces a fundamental signal-to-noise bottleneck: targeting specific grid points dilutes the expected signal by a factor of $1/M$, while querying broad regions accumulates excessive baseline noise from $D_0$. Formally, the per-query KL divergence between the Bernoulli response distributions is bounded by $O\big(\frac{\eps^2}{M\sigma^2}\big)$, which is smaller than the counterpart $O(\eps^2/\sigma^2)$ in the unquantized setting. Applying the chain rule for KL divergence over $n$ adaptive rounds forces the sample complexity to scale as $n = \Omega\big(M \cdot \frac{\sigma^2}{\eps^2} \log\frac{1}{\delta}\big) = \Omega\big(\frac{\sigma^2}{\eps^2} \log\frac{\sigma}{\eps} \log\frac{1}{\delta}\big)$.

Crucially, the above argument breaks down in the lighter-tailed regime $(k > 2)$. This stems from the fact that  the contribution from $\pm x_j$ scales as $q_j x_j^k = \Theta(\frac{\sigma^2}{M} x_j^{k-2})$; then, because $x_j$ grows exponentially, the sum $\sum_{j=1}^M x_j^{k-2}$ diverges rapidly when $k > 2$. Thus, to keep the $k$-th central moment bounded by $\sigma^k$, the grid must be truncated to $M = O(1)$ points.  This means that the ``logarithmic hiding space'' vanishes, and the lower bound reverts to $\Omega\big(\frac{\sigma^2}{\eps^2} \log\frac{1}{\delta}\big)$.

\section{Variations and Refinements}
\label{sec: variations and refinements}

The main estimator developed in Section~\ref{sec: algorithm} operates under a specific set of conditions: it assumes the target accuracy~$\eps$ and scale parameter~$\sigma$ are known \emph{a priori}, it utilizes $O\big(\log \frac{\lambda}{\sigma} + \log \frac{1}{\delta}\big)$ rounds of sequential adaptivity, and it is restricted to univariate distributions. In this section, we present four algorithmic extensions that relax these constraints.  Specifically, Section~\ref{sec: unknown eps} modifies the protocol to operate without a prespecified target accuracy, yielding an ``anytime'' estimator that adapts to an unknown sampling budget. Section~\ref{sec: partially unknown scale} details an approach for adapting to an unknown scale parameter~$\sigma$ when only loose bounds are available.  Section~\ref{sec: two-stage} demonstrates how trading the threshold-query model for general 1-bit queries can drastically reduce the interaction down to just two stages of adaptivity while maintaining order-optimal sample complexity. Section~\ref{sec:m_samples_per_query} extends the framework to a setting where
each 1-bit query may depend on a local batch of $m$ samples, showing that local
averaging proportionally reduces the refinement communication cost. Finally, Section~\ref{sec:multivar} outlines some implications of our results for multivariate mean estimation in $\mathbb{R}^d$.

\subsection{Unknown Target Accuracy}\label{sec: unknown eps}

Our estimator as described in Section~\ref{sec: algorithm} takes the target accuracy $\eps$ as an input and consumes a bounded number of samples.  Conversely, if a fixed sampling budget $n$ is pre-specified, one can invert the sample complexity bound in~\eqref{eq: scaling law of n} to determine the best achievable target accuracy.  Running the estimator with this parameter
naturally yields an estimate at the corresponding oracle accuracy while respecting the budget.

We now consider the setting where the true sampling budget $n_{\rm true}$ is not known to the learner in advance. This arises, for instance, when the communication horizon is uncertain, clients dynamically drop out, or the learner is allowed to keep refining its estimate until a resource constraint is reached. The parameters $k,\delta,\lambda$, and $\sigma$ remain known.

There are two natural versions of this problem. In the first, the budget is \emph{externally determined}: it is fixed in advance, or random but independent of the observed transcript.  In this case, the final refinement round is independent of the data, so it suffices to guarantee accuracy at that one round. In the second, stronger version, the stopping rule may be \emph{transcript-dependent}; for example, the learner may stop when consecutive estimates appear stable or when an estimate crosses a decision
threshold.  In this anytime-valid setting, a guarantee for any single pre-specified refinement round does not automatically apply to the data-dependent selected round. We therefore use a summable confidence schedule to ensure that all completed estimates are \textit{simultaneously} accurate
on one high-probability event.

Let
\[
  n_{\rm loc}(\delta_{\rm loc}, \lambda,\sigma)
  =
  \Theta \left(
  \log\frac{\lambda}{\sigma}+\log\frac{1}{\delta_{\rm loc}}
  \right)
\]
denote the number of samples used in one localization subroutine (Step 1 of Section~\ref{sec: algorithm}). For a generic target accuracy $\alpha \in (0,\sigma]$ and a generic refinement failure probability $\eta \in (0, \delta]$, let
$n_{\rm ref}(\alpha,\eta,k,\sigma)$ denote the number of
samples used by one invocation of the refinement subroutine
(Steps~2--6 of Section~\ref{sec: algorithm}) with target accuracy
$\alpha$ and failure probability at most $\eta$, which satisfies the scaling
\begin{equation}
  \label{eq: n_ref}
  n_{\rm ref}(\alpha,\eta, k,\sigma)
  =
    \begin{cases}
      \Theta_k \left(
      \dfrac{\sigma^2}{\alpha^2}
      \log\left(\dfrac{1}{\eta}\right)
      \right) & \text{if } k>2,
      \\ \\
      \Theta \left(
      \dfrac{\sigma^2}{\alpha^2}
      \log\left(\dfrac{\sigma}{\alpha}\right)
      \log\left(\dfrac{1}{\eta}\right)
      \right) & \text{if } k=2,
      \\ \\
      \Theta_k \left(
      \left(\dfrac{\sigma}{\alpha}\right)^{\frac{k}{k-1}}
      \log\left(\dfrac{1}{\eta}\right)
      \right) & \text{if } k\in(1,2).
    \end{cases}
\end{equation}

We first describe the estimator for the stronger anytime-valid guarantee. Split the failure probability as $\delta_{\rm loc} = \delta/2$ and $\delta_{\rm ref} = \delta/2$. The estimator first runs the localization subroutine once with parameters $(\delta_{\rm loc}, \lambda, \sigma)$ and consumes $n_{\rm loc}(\delta_{\rm loc},\lambda,\sigma)$ samples.
Then, for each refinement round $\tau=1,2,\dotsc$, it runs the
refinement subroutine  with progressively tightened parameters
\begin{equation}
\label{eq: parameter in round tau}
    ( \eps_{\tau}, \delta_{\tau}, k, \sigma)
  \quad \text{where }  \quad
  \text{target accuracy }\eps_{\tau} = \frac{\sigma}{2^{\tau}}
  \quad \text{and} \quad
  \text{confidence level } \delta_\tau
    =
    \frac{6\delta_{\rm ref}}{\pi^2\tau^2}
    =
    \frac{3\delta}{\pi^2\tau^2}.
\end{equation}
Each round $\tau$ consumes
$n_{\rm ref}(\eps_\tau,\delta_\tau,k,\sigma)$ additional samples. The estimator proceeds to the next round as long as completing that round
would not exceed the realized true sampling budget $n_{\rm true}$. 
When the true sampling budget is exhausted, the estimator outputs the last fully computed estimate $\hat{\mu}_T$, where 
\begin{equation}
  \label{eq: last round T}
  T = \max_{\tau \ge 1} \left\{ \sum_{s=1}^{\tau}
    n_{\rm ref}(\eps_s, \delta_s, k, \sigma) \le
  n_{\rm true} - n_{\rm loc}(\delta_{\rm loc}, \lambda, \sigma)\right\}
\end{equation}
is the final round where the refinement subroutine is completed. The realized budget $n_{\rm true}$, and hence $T$, may be transcript-dependent. 

To evaluate the optimality of this anytime estimator, we compare $\eps_T$ against the ``oracle accuracy'' $\eps^*= \eps^*(n_{\rm true})$, defined implicitly with respect to the realized budget by 
\begin{equation}
\label{eq: oracle_eps}
    n_{\rm ref}(\eps^*,\delta_{\rm ref},k,\sigma)
    =
    n_{\rm true}
    -
    n_{\rm loc}(\delta_{\rm loc},\lambda,\sigma).
\end{equation}
This is the ``optimal'' target accuracy achievable by a single refinement run, if the realized budget $n_{\rm true}$ were known in advance and the final estimate needed only confidence $\delta_{\rm ref}$.
Under a mild assumption that $n_{\rm true}$ is sufficiently large to complete the localization step and the first refinement round, we show that our anytime estimator matches the oracle accuracy up to a doubly-logarithmic factor.
\begin{theorem}\label{thm: unknown eps}
  Consider the preceding setup, and assume that the realized budget satisfies, almost surely,
   \[
    n_{\rm true}
    \ge
    n_{\rm loc}(\delta_{\rm loc},\lambda,\sigma)
    +
    n_{\rm ref}(\eps_1,\delta_1,k,\sigma).
    \]
    Then the final output $\hat\mu_T$ of the estimator described above satisfies
    \[
        \Pr\left(|\hat\mu_T-\mu|\le \eps_T\right)\ge 1-\delta
    \]
    even if the stopping rule is transcript-dependent.
    Furthermore, we have
    \[  
    \eps_T = O_k\left( \eps^* \left( 1 + \frac{\log\log (\sigma/\eps^*)}{\log(1/\delta)} \right)^{\frac{1}{p}} \right)
    \]
  where $p = 2$ if $k \ge 2$, and $p = \frac{k}{k-1}$ if $k \in (1, 2)$.
\end{theorem}

\begin{remark}[Anytime-validity and Doubly-logarithmic Overhead]
The additional $\log\log$ factor in
Theorem~\ref{thm: unknown eps} is a standard price of requiring an
anytime-valid guarantee.  This is analogous to the classical law-of-the-iterated-logarithm phenomenon in time-uniform mean estimation (see, e.g., \cite{darling1967confidence, jamieson2014lil, howard2021time, duchi2024information}).
We do not claim that the precise form of this factor is unavoidable for the present problem, only that it reflects the standard overhead of anytime validity.
\end{remark}

The summable confidence schedule $(\sum_{\tau} \delta_{\tau} \le \delta_{\rm ref})$ in~\eqref{eq: parameter in round tau} constructs a single high-probability event on which localization succeeds and every completed refinement estimate is accurate.  Therefore, after this event is established, the stopping round $T$ may be transcript-dependent without invalidating the guarantee. If the budget is externally determined (i.e., fixed in advance or random but independent of the samples, query feedback, and internal randomness of the estimator), this simultaneous guarantee is unnecessary. The final round is then independent of the transcript, and  it suffices to guarantee accuracy of that one round. This removes the union bound over refinement rounds and eliminates the doubly-logarithmic factor.

\begin{corollary}
\label{cor: unknown eps external}
    Suppose that $n_{\rm true}$ is externally determined. 
    Run the same halving scheme $\eps_\tau=\sigma/2^\tau$, but use the constant refinement confidence level $\delta_\tau=\delta_{\rm ref}$ for every  $\tau \ge 1$.
    If 
    \[
        n_{\rm true}
    \ge
    n_{\rm loc}(\delta_{\rm loc},\lambda,\sigma)
    +
    n_{\rm ref}(\eps_1,\delta_{\rm ref},k,\sigma),
    \]
    then
    \[
        \Pr\left(|\hat\mu_T-\mu|\le \eps_T\right)\ge 1-\delta
        \quad \text{and} \quad
        \eps_T = O_k(\eps^\ast),
    \]
    where $T$ is as defined in~\eqref{eq: last round T}, but under this constant-confidence schedule $\delta_\tau=\delta_{\rm ref}$.
\end{corollary}
The crux of both results is that, across all tail regimes $k>1$, the refinement complexity grows at least quadratically in $1/\eps$. Thus the sample costs of successive halving rounds grow geometrically, and the total number of samples spent before any round is dominated by the cost of the last round. The full derivation is provided in Appendix~\ref{appendix: unknown target accuracy}.

\subsection{Adapting to Unknown Scale $\sigma$}
\label{sec: partially unknown scale}
The sample complexity of our main mean estimator, as established in Theorem~\ref{thm: main}, scales with the ratio $\sigma/\eps$, where~$\sigma^k$ is a known upper bound on the true $k$-th central moment $\sigma_{\mathrm{true}}^k = \EE[|X-\mu|^k]$. This scaling is not ideal when the provided bound is highly conservative (i.e., $\sigma \gg \sigma_{\mathrm{true}}$). This contrasts with the unquantized setting, where there exist finite-variance mean estimators whose sample complexities scale with the true ratio $\sigma_{\mathrm{true}}/\eps$ without requiring any prior knowledge of $\sigma_{\mathrm{true}}^2$ \cite{lee2022optimal}.

Under the 1-bit communication constraint, it is difficult to learn $\sigma_{\mathrm{true}}$ and estimate the mean simultaneously. We consider a setting where both the target accuracy $\eps$ and the true scale parameter $\sigma_{\mathrm{true}}$ are unknown to the learner, but the learner is given a valid but potentially vast range $\sigma_{\mathrm{true}} \in [\sigma_{\min}, \sigma_{\max}]$ and seeks to estimate the mean to a \emph{relative accuracy} $\eps = r \sigma_{\mathrm{true}}$ for some known target ratio $r \in (0, 1)$. That is, we seek accuracy to within a given factor of the true (but unknown) scale parameter.

To achieve this, we wrap our main estimator in a geometric grid-search procedure.  The learner tests a sequence of logarithmically decaying guesses for the scale parameter defined by $\sigma_i = \sigma_{\max} \cdot 2^{-i}$ for $i = 0, 1, \dots, T$, where the maximum index is given by
\begin{equation}
  \label{eq: T = log(sigma_max/sigma_min)}
    T = \lceil \log_2(\sigma_{\max}/\sigma_{\min}) \rceil.
\end{equation}
 For each guessed scale $\sigma_i$, the learner runs the mean estimator in Section~\ref{sec: algorithm} with parameters
\begin{equation}
  \label{eq: eps_i = r sigma_i/5}
  \left(\eps_i, \delta_i, \lambda, \sigma_i \right)=
  \left(
    \frac{r \sigma_i}{6},
    \quad \frac{\delta}{T+1},
    \quad
    \lambda,
    \quad
    \frac{\sigma_{\max}}{2^i}
  \right)
\end{equation}
to obtain a candidate mean estimate $\hat{\mu}^{(i)}$ and an associated confidence interval
\begin{equation}
  \label{eq: CI}
  I_i = [ \hat{\mu}^{(i)} \pm \eps_i]
  = \left[ \hat{\mu}^{(i)} -  \frac{r \sigma_i}{6}, \hat{\mu}^{(i)} +
  \frac{r \sigma_i}{6} \right].
\end{equation}
 The algorithm proceeds sequentially and halts at the first index $i$ for which the newly generated confidence interval fails to intersect with \emph{any} of the previously established intervals: 
 \begin{equation}
     \label{eq: feasible condition}
     I_i \cap I_l = \emptyset \quad \text{ for some } l < i.
 \end{equation} 
 It then outputs the estimate $\hat{\mu}^{(i^*)}$ from the last successful index $i^* = i-1$. If no halting occurs,  then the algorithm sets $i^*=T$ and outputs $\hat\mu^{(T)}$.
 
Because the target accuracy scales proportionally with the guessed scale (i.e., $\eps_i = r \sigma_i / 6$), the ratio $\sigma_i / \eps_i = 6/r$ remains constant across all rounds $i$. Consequently, the sample complexity of the refinement phase for \emph{every} grid point depends only on the relative accuracy $r$ (and the error probability $\delta_i$). Summing the sample complexities across all $T+1$ grid points yields the following performance guarantee.

\begin{theorem}[Adaptation to Unknown Scale]\label{thm: partially unknown scale}
  Given $r \in (0, 1)$ and $\sigma_{\mathrm{true}} \in [\sigma_{\min}, \sigma_{\max}]$ for some $0 < \sigma_{\min} < \sigma_{\max} \le \lambda$, the adaptive 1-bit mean estimator described above is $(r\sigma_{\mathrm{true}}, \delta)$-PAC for any distribution in $\mathcal{D}(k, \lambda, \sigma_{\mathrm{true}})$. The total sample complexity is bounded by
  \begin{equation*}
      n = O_k \left( \log\left(\frac{\sigma_{\max}}{\sigma_{\min}}\right) \cdot \left( N_k(r) \cdot\log\left(\frac{\log(\sigma_{\max}/\sigma_{\min})}{\delta}\right) + \log\left(\frac{\lambda}{\sqrt{\sigma_{\min}\sigma_{\max}}}\right) \right) \right)
  \end{equation*}
  where $N_k(r)$ captures the asymptotic sample complexity scaling with respect to the fixed ratio $r$:
  \begin{equation*}
      N_k(r) = 
      \begin{cases} 
      \dfrac{1}{r^2} & \text{if } k > 2 
      \\ \\
      \dfrac{1}{r^2} \log\left(\dfrac{1}{r}\right) & \text{if } k = 2 
      \\ \\
      \left(\dfrac{1}{r}\right)^{\frac{k}{k-1}} & \text{if } k \in (1, 2).
      \end{cases}
  \end{equation*}
\end{theorem}
The proof is given in Appendix~\ref{appendix: partially unknown scale}.  
Observe that when $\sigma_{\max} = \Theta(\sigma_{\min})$, the scaling in Theorem~\ref{thm: partially unknown scale} simplifies to that of Theorem \ref{thm: main} (i.e., the case of known $\sigma$), and more generally the penalty is mild, as the most significant difference is the multiplication by $\log\frac{\sigma_{\max}}{\sigma_{\min}}$ (i.e., a logarithmic penalty).

\subsection{Two-Stage Variant}\label{sec: two-stage}

Our mean estimator in Section~\ref{sec: algorithm} uses $O\left(\log \frac{\lambda}{\sigma}+\log\frac1\delta\right)$ rounds of adaptivity. Specifically, the localization step (Step~1 of Section~\ref{sec: algorithm}), which performs median estimation through noisy binary search, introduces sequential dependencies that require $O\left( \log \frac{\lambda}{\sigma} + \log \frac{1}{\delta} \right)$ rounds of adaptivity.
In contrast, once an interval of length $O(\sigma)$ containing the mean has been identified, the refinement procedure (Steps~2--6 of Section~\ref{sec: algorithm}) can be executed in one additional non-adaptive round.

In this section, we replace the sequential localization step by a
non-adaptive localization procedure.  This yields a two-stage estimator: the first stage localizes the mean using general 1-bit queries, and the second stage runs the refinement procedure from Section~\ref{sec: algorithm}.  However, this comes at the cost of using general 1-bit queries in the first stage (localization), as opposed to using only threshold queries throughout.

The localization idea is coding-theoretic. We partition the search range
$[-\lambda,\lambda]$ into $O(\lambda/\sigma)$ bins of width $O(\sigma)$ and
fix in advance a deterministic binary codebook assigning one codeword to each
bin. For the $t$-th sample, the learner sends a general 1-bit query that
returns the $t$-th coordinate of the codeword corresponding to the bin in
which the sample falls. The received bit string is then decoded by the nearest
neighbor rule in Hamming distance, and the learner returns a constant-size
enlargement of the decoded bin.

Intuitively, since $\EE|X-\mu|\le\sigma$, Markov's inequality implies
that most samples fall within an $O(\sigma)$-neighborhood of the mean. Thus
the received word is, in aggregate, closer to the codeword of a nearby bin
than to the codeword of any far-away bin. A globally redundant codebook and
nearest-neighbor decoding allow this comparison to hold simultaneously over
all far bins with only
\[
    O\left(
        \log\frac{\lambda}{\sigma}
        +
        \log\frac{1}{\delta_{\rm loc}}
    \right)
\]
samples. The formal localization procedure and proof are given in
Appendix~\ref{appendix: two-stage}.

\begin{theorem}[Non-adaptive Localization] \label{thm: alternative localization}
Fix $\lambda \ge \sigma$ and $\delta_{\rm loc} \in (0, 1/2)$. There exists a deterministic non-adaptive 1-bit localization protocol taking $(\delta_{\rm loc},\lambda,\sigma)$ as input such that, for every distribution $D$
with mean $\mu\in[-\lambda,\lambda]$ and $\EE |X-\mu|\le \sigma$, it returns an interval~$I$ of length $|I| = O(\sigma)$ containing the true mean~$\mu$ with probability at least $1-\delta_{\rm loc}$. The number of samples used is $O\left( \log \left( \frac{\lambda}{\sigma}\right) + \log \frac{1}{\delta_{\rm loc}}\right)$.
\end{theorem}

\begin{remark}[$k$-independence]
     Similar to the localization procedure in Section~\ref{sec: algorithm}, the non-adaptive localization protocol does not require knowing~$k$. The assumption $\EE|X-\mu|\le\sigma$ is weaker than the bounded $k$-th central moment assumption used throughout the paper. Indeed, for any $D \in \mathcal D(k,\lambda,\sigma)$ with $k>1$, Lyapunov's inequality gives
    \[
        \EE |X-\mu|
        \le
        \left(\EE|X-\mu|^k\right)^{1/k}
        \le \sigma.
    \]
    Consequently, Theorem~\ref{thm: alternative localization} applies uniformly to every $D\in\mathcal D(k,\lambda,\sigma)$ with $k>1$, without taking $k$ as input. 
\end{remark}

The earlier two-stage variant in our preliminary version
\cite[Section~4.5]{lau2025sequential} used a \textit{coordinatewise} Gray-code localization procedure adapted from the Gaussian mean estimation procedure of~\cite{cai2024distributed}. In the finite-moment setting, this procedure estimated the Gray-code bits coordinatewise and amplified each bit by majority voting. In contrast, the procedure in Theorem~\ref{thm: alternative localization} replaces coordinatewise bit recovery by a globally redundant codebook and a nearest-neighbor decoder. This global decoding step removes the suboptimal multiplicative overhead arising from bitwise confidence amplification, improving the localization cost from $O\left(\log\frac{\lambda}{\sigma} \cdot \log\frac{\log(\lambda/\sigma)}{\delta_{\rm loc}} \right)$ to $O\left( \log\frac{\lambda}{\sigma} + \log\frac{1}{\delta_{\rm loc}} \right)$.

By replacing the localization procedure of our main estimator with the non-adaptive localization procedure in Theorem~\ref{thm: alternative localization} (taking $\delta_{\rm loc} = \delta/2)$, we obtain a two-stage mean estimator with the same order-optimal sample complexity.\footnote{The only modification needed in the refinement analysis is a constant-factor one: Theorem~\ref{thm: alternative localization} returns an interval of length $C\sigma$ for some universal constant $C$, rather than the
specific constant $8$ used in Step~1 of Section~\ref{sec: algorithm}. After recentering at the midpoint of this interval and adjusting the numerical constants in the cutoff threshold, the refinement proof is unchanged.}

\begin{corollary}
Fix $k > 1$, as well as $\lambda > \sigma > \eps > 0$, and $\delta \in (0, 1/2)$. 
The two-stage mean estimator described above is $(\eps, \delta)$-PAC for the distribution family $\mathcal{D}(k, \lambda, \sigma)$, with sample complexity 
\begin{align}
    n
    =
    O\left( \log \left( \frac{\lambda}{\sigma}\right) \right) +
    \begin{cases}
      O_k \left(  \dfrac{\sigma^2}{\eps^2} \cdot
      \log\left(\dfrac{1}{\delta}\right) \right) & \text{if } k > 2
      \\ \\
      O \left(  \dfrac{\sigma^2}{\eps^2} \cdot
        \log \left( \dfrac{\sigma}{\eps}\right)  \cdot \log\left(
      \dfrac{1}{\delta}\right) \right) &
      \text{if } k = 2
      \\ \\
      O_k \left(  \left( \dfrac{\sigma}{\eps}  \right)^{\frac{k}{k-1}}  
      \cdot \log\left(\dfrac{1}{\delta}\right) \right)  & \text{if
      } k \in (1, 2),
    \end{cases}
  \end{align}
  where $O_k(\cdot)$ represents $O(\cdot)$ notation with a hidden constant that depends on $k$. Furthermore, the estimator uses only two stages of adaptivity: a deterministic non-adaptive general-query localization stage, followed by the randomized threshold-query refinement stage (see Steps~2--6 of Section~\ref{sec: algorithm}).
\end{corollary}

\subsection{Multiple Samples per 1-Bit Query}\label{sec:m_samples_per_query}
We now consider an extension in which each 1-bit query may depend on a
local batch of $m$ fresh observations, while still returning only a
single bit. This models situations where local devices (e.g., sensors or
edge nodes) can accumulate a batch of observations before transmission,
or where communication occurs at a coarser time scale than data acquisition.
Formally, in round $t$, the agent observes $m$ i.i.d. samples
$X_{t,1},\ldots,X_{t,m}\sim D$. The learner sends a 1-bit quantization
function
\[
    Q_t \colon \mathbb{R}^m\to\{0,1\},
\]
and receives the single bit
\[
    Y_t = Q_t(X_{t,1},\ldots,X_{t,m}).
\]
Thus, each 1-bit query still communicates exactly one bit, but that bit may
be computed from $m$ local samples rather than from one raw observation.

To obtain a simple guarantee, we use a local averaging strategy: for each query, the local batch is first averaged as
\[
    \widebar X_t = \frac1m\sum_{j=1}^m X_{t,j},
\]
and the learner applies the estimator of Section~\ref{sec: algorithm} to the
induced distribution of $\widebar X_t$. Equivalently, each 1-bit query in this
construction has the form
\[
    Q_t(X_{t,1},\ldots,X_{t,m})
    =
    \widetilde Q_t\!\left(\widebar X_t\right)
\]
for some scalar threshold query $\widetilde Q_t \colon \mathbb R\to\{0,1\}$.
The key observation is that $\EE[\widebar X_t] = \EE[X] = \mu$, but $\widebar X_t$ has a smaller effective moment scale. Specifically, we have
\begin{equation}\label{eq:sigma_m}
    \EE|\widebar X_t-\mu|^k
    \le \sigma_m^k
    \quad \text{where} \quad
    \sigma_m \coloneqq
    \begin{cases}
    C_k \sigma \cdot m^{-1/2}, & k \ge 2, 
    \\
    C_k \sigma \cdot  m^{-(1-1/k)}, & 1<k<2,
    \end{cases}
\end{equation}
for a constant $C_k$ depending only on $k$. The case $1<k<2$ follows from the von Bahr-Esseen inequality~\cite{von1965inequalities}, while the case $k\ge 2$ follows from Marcinkiewicz–Zygmund inequality (see e.g.~\cite[Exercise 6.2.6]{vershynin2026high}).

Consequently, we may apply the estimator of Section~\ref{sec: algorithm} directly to the induced distribution of $\widebar X_t$, with the scale parameter~$\sigma$ replaced by the effective scale parameter~$\sigma_m$.
Since each query on $\widebar X_t$ corresponds to one communicated bit (and $m$ drawn samples), Theorem~\ref{thm: main} gives the following guarantee on the bit complexity and sample complexity.

\begin{corollary}[Multiple Samples per Query]\label{cor:m_samples_per_query}
Fix $k > 1$, and let $\lambda \ge \sigma \ge \eps > 0$, and $\delta \in (0, 1/2)$.
Suppose each 1-bit query may depend on $m$ i.i.d. samples from a distribution~$D\in\mathcal{D}(k,\lambda,\sigma)$, and consider the local-averaging construction above. Assume $\sigma_m \ge \eps$, where $\sigma_m$ is as defined in~\eqref{eq:sigma_m}. Then the estimator described above is $(\eps,\delta)$-PAC with a bit complexity (i.e., number of 1-bit queries) of
\[
    n_{\mathrm{bit}}
    = O\left(
        \log\left(\frac{\lambda}{\sigma_m}\right)
    \right)
    +
    \begin{cases}
    O_k\left(
        \dfrac{\sigma_m^2}{\eps^2} \cdot
        \log\left(\dfrac{1}{\delta}\right)
    \right), & k>2,\\ \\
    O\left(
        \dfrac{\sigma_m^2}{\eps^2} \cdot 
        \log\left(\dfrac{\sigma_m}{\eps}\right) \cdot
        \log\left(\dfrac{1}{\delta}\right)
    \right), & k=2,\\ \\
    O_k\left(
        \left(\dfrac{\sigma_m}{\eps}\right)^{\frac{k}{k-1}} \cdot
        \log\left(\dfrac{1}{\delta}\right)
    \right), & 1<k<2.
    \end{cases}
\]
and a total sample complexity (i.e., $n = m \cdot n_{\mathrm{bit}}$) of
\begin{equation}\label{eq:sample_complexity_m_samples_per_query}
    n
    =
    \begin{cases}
    O_k\left(
        m\log\left(\dfrac{\lambda\sqrt m}{\sigma}\right)
        + \dfrac{\sigma^2}{\eps^2}
          \log\left(\dfrac{1}{\delta}\right)
    \right), & k>2,\\ \\
    O\left(
        m\log\left(\dfrac{\lambda\sqrt m}{\sigma}\right)
        + \dfrac{\sigma^2}{\eps^2}
          \log\left(\dfrac{\sigma}{\eps\sqrt m}\right)
          \log\left(\dfrac{1}{\delta}\right)
    \right), & k=2,\\ \\
    O_k\left(
        m \log\left(\dfrac{\lambda m^{1-1/k}}{\sigma}\right)
        +  \left(\dfrac{\sigma}{\eps}\right)^{\frac{k}{k-1}}
          \log\left(\dfrac{1}{\delta}\right)
    \right), & 1<k<2.
    \end{cases}
\end{equation}
\end{corollary}

The corollary shows that local averaging reduces the refinement communication cost by roughly a factor of $m$ across all moment regimes, while only slightly increasing the localization communication cost by an additive $\log m$ term. More precisely:
\begin{itemize}
  \item \textbf{Bit complexity.} The number of bits/queries needed for refinement drops by a factor of $m$ because the effective scale $\sigma_m$ is smaller.
  
  \item \textbf{Sample complexity.} The total number of samples drawn
        $n = m \cdot n_{\mathrm{bit}}$ stays essentially unchanged for the refinement procedure when $k \neq 2$ since the factor $m$ cancels out. For $k=2$, the logarithmic penalty slightly improves from $\log(\sigma/\epsilon)$ to $\log(\sigma/(\epsilon\sqrt m))$.
        The localization raw cost, however, grows proportionally to $m$.
        
  \item \textbf{Saturation threshold.} As long as $m = O\big((\sigma/\epsilon)^2\big)$ for
        $k \ge 2$ and $m = O_k\big((\sigma/\epsilon)^{\frac{k}{k-1}}\big)$ for $1 < k < 2$, the effective scale satisfies $\sigma_m \ge \epsilon$ and the above
        bounds apply directly. When $m$ exceeds this saturation threshold, the condition $\sigma_m \ge \epsilon$ fails. In the regime $\sigma_m \ll \epsilon$, the learner can simply treat the estimation problem as a localization problem using an inflated effective scale parameter of $\eps > \sigma_m$.\footnote{This follows from a straightforward modification of Step 1 of Section~\ref{sec: algorithm}. The learner first localizes the median of $\bar{X}_t$ to an interval $[L, U]$ of length $|U-L| \le c_k \eps$ for some sufficiently small constant $c_k$, using $O(\log(\lambda/\epsilon) + \log(1/\delta))$ queries, Then the learner forms an enlarged interval $I = [L - \sigma_m, U + \sigma_m]$ containing $\EE[X_t] = \EE[X] = \mu$ with high probability. Taking the midpoint of $I$ as the final estimate gives an $\eps$-accurate estimate since $0.5|I| \le \eps$.} This yields bit complexity
        $O\big(\log(\lambda/\epsilon) + \log(1/\delta)\big)$ and sample complexity $O\big(m\log(\lambda/\epsilon) + m \log(1/\delta) \big)$. Thus, increasing the number of samples per 1-bit query beyond the saturation point does not further reduce the communication cost, but instead increases the number of samples used.
\end{itemize}

A detailed derivation of the moment bounds~\eqref{eq:sigma_m} and the subsequent substitutions can be found in Appendix~\ref{app:m_samples_per_query}.

\subsection{Multivariate Mean Estimation}\label{sec:multivar}

The multivariate case (i.e., $X \in \mathbb{R}^d$ with $d > 1$) is naturally of significant interest. We have focused our analysis exclusively on the univariate setting up to this point, as it forms the necessary theoretical foundation and already presents substantial challenges.  Nevertheless, our 1-bit architecture readily provides constructive, baseline guarantees for higher dimensions. To maintain clarity and isolate the impact of dimensionality, we restrict our discussion in this subsection to the canonical finite-variance setting ($k=2$).

Specifically, suppose that $X$ takes values in $\mathbb{R}^d$, where each coordinate $X_j$ (for $j=1,\dotsc,d$) individually satisfies our earlier assumptions for $k=2$; namely, a bounded mean $\EE[X_j] \in [-\lambda, \lambda]$ and bounded variance $\operatorname{Var}(X_j) \le \sigma^2$. By applying our univariate estimator coordinate-wise with a refined target accuracy of $\eps/\sqrt{d}$ and a union-bounded confidence parameter of $\delta/d$, we guarantee that every coordinate is estimated to within $\eps/\sqrt{d}$ accuracy. Consequently, we obtain an overall multivariate estimate that is $\eps$-accurate in the $\ell_2$ norm with probability at least $1-\delta$. In accordance with Theorem~\ref{thm: main}, the total sample complexity across all $d$ coordinates is
\[
\widetilde{O}\left( \frac{d^2 \sigma^2}{\eps^2}\log\frac{1}{\delta} + d\log\frac{\lambda}{\sigma}\right),
\]
where the $d^2$ factor arises from (i) using the scaled accuracy parameter $\eps/\sqrt{d}$, and (ii) running the univariate subroutine~$d$ times.  This may seem potentially loose on first glance, due to the correct scaling being $\frac{\sigma^2}{\eps^2} \cdot \left(d + \log(1/\delta) \right)$ in the absence of a communication constraint \cite{lugosi2019sub}.
However, under 1-bit feedback, the $d^2 \sigma^2 / \eps^2$ dependence in fact unavoidable even in the special case of Gaussian random variables; see \cite[Theorem 8]{cai2024distributed} with the parameter $m'$ therein equating to $n/d$ in our notation under 1-bit feedback.\footnote{To give slightly more detail, the parameters $m$ and $b_i = 1$ therein equate respectively to the number of samples $n$ and number of bits allowed per feedback.}  
Moreover, if the communication bottleneck is relaxed to allow $d$ bits of feedback per sample (i.e., one bit \emph{per coordinate}), applying our univariate estimator coordinate-wise yields a sample complexity of $\widetilde{O}\big( \frac{d \sigma^2}{\eps^2} \log(1/\delta) + d \log(\lambda/\sigma) \big)$. In the constant error probability regime ($\delta = \Theta(1)$), this matches the unconstrained rate up to logarithmic factors.   
In the regime $\delta = o(1)$ there remains a significant gap due to the fact that $d \log \frac{1}{\delta} \gg d + \log\frac{1}{\delta}$, but this gap is inherent to any approach that controls each coordinate's error to $O\big( \frac{\eps}{\sqrt d} \big)$ separately.

Beyond the issue of joint $(d,\delta)$ dependence, another limitation of the coordinate-wise approach is that it does not capture the dependence on off-diagonal terms in the covariance matrix $\Sigma$.  Doing so may be significantly more difficult, particularly when $\Sigma$ is not known exactly and so ``whitening'' techniques cannot readily be used.  We leave such considerations for future work.

\section{Conclusion}
In this paper, we studied the fundamental limits of mean estimation under the extreme constraint of 1-bit communication. We proposed a novel adaptive estimator based solely on threshold queries that is $(\eps, \delta)$-PAC across a broad non-parametric family of distributions with
bounded mean and bounded $k$-th central moment. Crucially, we established that this estimator achieves order-optimal sample complexity across all tail regimes $k > 1$. For $k \neq 2$, its sample complexity matches the unquantized minimax rates up to an unavoidable additive $O(\log(\lambda/\sigma))$ localization cost. For the finite-variance case ($k=2$), we established a novel information-theoretic lower bound proving that the extra multiplicative $O(\log(\sigma/\eps))$ penalty is an inescapable consequence of the 1-bit communication constraint, confirming order-optimality for $k=2$ (and hence across for all $k > 1$).

Beyond the main threshold-query estimator, we developed several algorithmic
variants that broaden the applicability of the framework. We showed how to
handle an unknown sampling budget through an anytime-valid estimator, and how
to adapt to an unknown scale parameter given only loose bounds. We also
provided a two-stage estimator by replacing the sequential threshold-query
localization step with a non-adaptive coding-theoretic procedure, yielding an
order-optimal 1-bit mean estimator with only two stages of adaptivity when
general 1-bit queries are allowed. Finally, we studied a setting in which each
1-bit query may depend on a local batch of $m$ samples, showing that local
averaging can proportionally reduce the refinement communication cost across
all tail regimes.

Our lower bounds complement these algorithmic results. In addition to the
finite-variance quantization penalty, we established an adaptivity gap for
threshold queries and even for more general interval queries: non-adaptive
strategies must incur sample complexity scaling linearly with the search-size
ratio $\lambda/\sigma$, whereas our adaptive estimators depend only
logarithmically on this ratio. This separation highlights the central role of
adaptivity in non-parametric 1-bit mean estimation under simple query models.

Several directions remain for future work, including achieving fully
parameter-free adaptation to both the scale and target accuracy under minimal
assumptions, characterizing what is achievable with fully non-adaptive general
1-bit queries, and extending the 1-bit quantization framework to multivariate
settings beyond the coordinate-wise approach.

\section*{Acknowledgement}
This work was supported by the Singapore National Research Foundation (NRF) under its AI Visiting Professorship programme.

\bibliography{bibliography}
\clearpage

\appendix
{\Huge \bf \centering Appendix \par}
\section{Proof of Theorem~\ref{thm: main} (Performance Guarantee of 1-bit Mean Estimator)}\label{appendix: proof of main result}
We first state a useful generalization of Chebyshev's inequality that
will be used multiple times in the proof.
\begin{lemma}[$k$-moment Chebyshev's Inequality]\label{lem: k-moment chebyshev}
  Suppose that the random variable $X$ has a finite $k$-th central
  moment bounded by $\sigma^k$
  for some $k > 1$, i.e., $\EE\big[ |X - \mu |^k \big] \le \sigma^k$.
  Then, for each $t > 0$, we have
  \begin{equation}
    \label{eq: k-moment chebyshev}
    \Pr \left(|X - \mu | \ge t \right)
    = \Pr \left(|X - \mu |^k \ge t^k \right)
    \le \frac{\EE\big[ |X - \mu |^k \big]}{t^k}
    \le \frac{\sigma^k}{t^k}
  \end{equation}
\end{lemma}
We proceed in several steps as we outlined in Section~\ref{sec: algorithm}.

\textbf{Step 1 (Narrowing Down the Mean via the Median):}
If $\lambda \le 4 \sigma$, the interval $[-\lambda, \lambda]$ already contains $\mu$ and has length at most $8 \sigma$. Hence we assume $\lambda \ge 4 \sigma$ for the remainder of the Step 1.
We discretize the interval $[-\lambda, \lambda]$ containing $\EE[X]$
into a discrete set of points with uniform spacing of
$\sigma$:\footnote{For ease of analysis, we assume that $\lambda$ is
an integer multiple of $\sigma$.}
\begin{equation*}
  \left\{-\lambda, -\lambda + \sigma, \dotsc, -\sigma, 0, \sigma,
  \dotsc, \lambda-\sigma, \lambda\right\}.
\end{equation*}
We then form estimates $L, U \in  \left\{-\lambda, -\lambda + \sigma,
\dotsc,  \lambda-\sigma, \lambda \right\}$ using noisy binary
search~\cite{gretta2023sharp} that satisfy
\begin{equation}
  \label{eq: median lower bound}
  \Pr\big(F(L) < 0.5 \text{ and } F(L+\sigma) > 0.49\big) \ge 1 -
  \frac{\delta_{\rm loc}}{2}
\end{equation}
and
\begin{equation}
  \label{eq: median upper bound}
  \Pr\big(F(U-\sigma) < 0.51 \text{ and } F(U) > 0.5\big) \ge 1 -
   \frac{\delta_{\rm loc}}{2}.
\end{equation}
The algorithm in~\cite{gretta2023sharp} achieves this using at most $O\big(\log
\frac{\lambda}{\sigma} +  \log \frac{1}{\delta_{\rm loc}}\big)$ 1-bit queries.
Under these high-probability events, the median $M$ satisfies $L \le M \le U$.
Using the well-known fact that the median minimizes the mean absolute
error~\cite[Theorem 4.5.3]{degroot2013probability}:
\begin{equation*}
  \EE  |X - M| \le \EE  |X - d| \quad \text{for each } d \in \mathbb{R},
\end{equation*}
we have
\[
  \left| \mu - M \right|
  =  \left| \EE[X] - M \right|
  =  \left| \EE[X - M] \right|
  \le  \EE  |X - M|
  \le  \EE  |X - \mu|.
\]
Meanwhile, applying Jensen's inequality to the convex function $z
\mapsto |z|^k$ (for $k > 1$), along with the $k$-th central moment bound, yields
\[
  {\left( \EE \left[ |X - \mu|  \right] \right)}^k
  \le
  \EE  \left[  |X - \mu|^{k} \right]
  \le \sigma^k
  \implies
  \EE  |X - \mu|  \le \sigma.
\]
Combining these two findings and $L \le M \le U$, we have
\begin{equation*}
  \left| \mu - M \right| \le \sigma \implies
  \mu \in [L -\sigma, U +\sigma].
\end{equation*}
Next, we bound the length of this localized interval,
$(U+\sigma)-(L-\sigma)$. We consider two cases:
(i)  $L + \sigma \ge U - \sigma$
and (ii) $L + \sigma < U - \sigma$.
In case (i), the interval length is trivially at most $4 \sigma$.
In case (ii), we claim
that the interval length is at most $8\sigma$ (or equivalently strictly less than $9 \sigma$ since $L$ and $U$ lie on a $\sigma$-spaced grid).
Seeking contradiction, suppose that $(U +\sigma) -
(L -\sigma) \ge 9\sigma$. Then we must have either
\begin{equation*}
  \mu - (L -\sigma) \ge 4.5 \sigma
  \quad \text{or} \quad
  (U + \sigma) - \mu  \ge 4.5\sigma.
\end{equation*}
We will show that $\mu - (L -\sigma) \ge 4.5 \sigma$ (which implies
$\mu - 2.5\sigma \ge L+ \sigma$) will lead to a contradiction; the
case $(U + \sigma) - \mu \ge 4.5 \sigma$ is similar.
Using~\eqref{eq: median lower bound}, we have
\begin{equation*}
  \Pr\left(X \le \mu - 2.5 \sigma \right) \ge
  \Pr(X \le L + \sigma) = F_X(L+ \sigma)   > 0.49.
\end{equation*}
On the other hand, by the ``$k$-moment'' Chebyshev's
inequality~\eqref{eq: k-moment chebyshev}, we have
\begin{equation*}
  \Pr \left(X \le \mu - 2.5\sigma \right)
  \le \Pr \left(|X - \mu | \ge 2.5\sigma \right)
  \le \frac{1}{2.5^k}
  < \frac{1}{2.5}
  < 0.49,
\end{equation*}
which is a contradiction.
Therefore, we have shown that with probability at least $1 - \delta_{\rm loc}$,
the mean $\mu$ lies in the interval $[L - \sigma,U + \sigma]$ of length at most $8\sigma$.

Throughout the refinement analysis (Steps~2--6), we condition on the localization success
event from Step~1 and work in the recentered coordinate system. Thus
$|\mu|\le 4\sigma$.

\textbf{Step 2 (Cutoff Threshold Selection):}
For $s>8\sigma$, define
\[
  A(s) \coloneqq \EE \left[ |X| \cdot \mathbf 1\{|X|\ge s\} \right].
\]
On the event $\{|X|\ge s\}$, the reverse triangle inequality gives
\[
   |X-\mu|\ge |X|-|\mu|\ge s-|\mu|.
\]
Combining this, the triangle inequality $|X| \le |X-\mu| + |\mu|$, and the $k$-moment Chebyshev's inequality~\eqref{eq: k-moment chebyshev} yields
\begin{equation}
\label{eq: tail contribution triangle inequality}
\begin{aligned}
     A(s)
  &\le  \EE \left[ |X-\mu| \cdot \mathbf{1} \left(|X| \ge s  \right) \right]
  + \big| \mu \big| \cdot \Pr \left(|X| \ge s  \right) \\
  &\le \EE\left[ |X-\mu| \cdot \mathbf 1\{|X-\mu|\ge s-|\mu|\} \right] 
  + |\mu| \cdot \Pr(|X-\mu|\ge s-|\mu|) \\
  &\le \EE\left[ |X-\mu| \cdot \mathbf 1\{|X-\mu|\ge s-|\mu|\} \right] 
  + |\mu| \cdot \frac{\sigma^k}{(s-|\mu|)^k}.
\end{aligned}
\end{equation}
For the first term, since $|X-\mu|\ge s-|\mu|$ on the relevant event, we have
\[
  |X-\mu|
  \le
  \frac{|X-\mu|^k}{(s-|\mu|)^{k-1}}.
\]
Substituting this bound into~\eqref{eq: tail contribution triangle
inequality}, we obtain
\begin{equation*}
  A(s)
  \le \frac{\sigma^k}{ { \left(s - |\mu| \right) }^{k-1}} + \big| \mu
  \big| \cdot \frac{\sigma^k}{ {(s-|\mu|)}^k}
  \le \frac{\sigma^k}{(s/2)^{k-1}} + 4\sigma\frac{\sigma^k}{(s/2)^k}
  \le \frac{8\sigma^k}{s^{k-1}},
\end{equation*}
where the constant is universal, using the operative reduction $k \le 3$ (see Remark~\ref{rem: operative moment}).

We choose
\[
t_0 = C_k \sigma\left(\frac{\sigma}{\eps}\right)^{1/(k-1)}
\]
where $C_k$ is sufficiently large so that
\[
    C_k \ge 8 \implies 8 t_0 > 8\sigma
    \qquad\text{and}\qquad
    C_k^{k-1} \ge 16^{\frac{1}{k-1}}  \implies \frac{8\sigma^k}{t_0^{k-1}}\le \eps/2.
\]
Since $A(s)$ is nonincreasing in $s$, every $s\ge t_0$ satisfies
$
   \EE \left[ |X| \cdot \mathbf{1} \{|X|\ge s\}\right]\le \eps/2.
$
In particular, for the final grid-aligned cutoff $t \ge t_0$ chosen in Step~3, we have
\[
   \left|\EE [X \cdot \mathbf{1} \{|X|\ge t\}]\right|
   \le \EE[|X| \cdot \mathbf 1\{|X|\ge t\}]
   \le \frac{\eps}{2}
\]
Thus, it remains to form a final high-probability estimate $\hat{\mu}$ of the
    ``clipped mean'' $ \EE \left[ X \cdot \mathbf{1} \left(|X| <
      t  \right) \right]$ satisfying    
    \[
      \left| \hat{\mu} - \EE\left[ X \cdot \mathbf{1} \left(|X| < t  \right) \right] \right| \le  \frac{\eps}{2},
    \]
    as this implies
    \begin{equation*}
        \left| \mu -  \hat{\mu} \right| 
        = \Big|
        \EE \left[ X \cdot \mathbf{1} \left(|X| \ge t  \right) \right]
        + \EE \left[ X \cdot \mathbf{1} \left(|X| < t  \right)
        \right]  -  \hat{\mu} \Big| 
        \le
        \Big|  \EE \left[ X \cdot \mathbf{1} \left(|X| \ge t  \right)
        \right] \Big| +
        \Big|  \EE \left[ X \cdot \mathbf{1} \left(|X| < t  \right)
        \right]  -  \hat{\mu}  \Big| 
        \le \eps.
    \end{equation*}

\textbf{Step 3 (Significant Region Partitioning).} 
Recall the symmetric regions with $R_i = \sigma[m_{i-1},m_i)$ for $i \ge 1$, the index $\imax$, and the final cutoff
$t = m_{\imax} \sigma$ from Step~3 of Section~\ref{sec: algorithm}.
By construction, we have $t\ge t_0$, so the tail bound from Step~2 applies at
the final cutoff $t$.
Since the union of the regions is $(-t, t)$ and the only overlap is at $x = 0$, we have
\[
  \sum_{1\le |i|\le i_{\max}} x \cdot\mathbf 1\{x\in R_i\}
  = x \cdot \mathbf 1\{|x|<t\}
\]
for every $x \in \mathbb{R}$.  Taking expectations gives
\[
  \sum_{1\le |i|\le i_{\max}}\mu_i
  = \sum_{1\le |i|\le i_{\max}} \EE[X \cdot \mathbf{1}\{X\in R_i\}]
  = \EE[X \cdot \mathbf{1} \{|X|<t\}].
\]
Therefore, to estimate the overall clipped mean, it is sufficient to independently estimate the local mean contribution~$\mu_i$ of each region $R_i$.

\textbf{Step 4 (Region-Wise Randomized Threshold Queries):}
We now analyze the threshold query procedure to form local estimate for the right-sided regions $R_i=[a_i,b_i)$ with $i \ge 1$. The left-sided case follows by applying the same argument to $Z=-X$.
Let $T_i \sim \text{Unif}(a_i, b_i)$ be a random threshold independent of $X$.
To formally justify the local mean identity introduced in Step 4 of Section~\ref{sec: algorithm}, we define a (hypothetical) stochastic quantizer $\mathrm{SQ}_i(X)$ that rounds $X$ to the boundaries of $R_i$ if $X$ falls within the region, and outputs $0$ otherwise: 
\begin{equation*}
  \mathrm{SQ}_i(X) = 
  \begin{cases} 
  a_i & \text{if } X \in R_i \text{ and } X < T_i \\ 
  b_i & \text{if } X \in R_i \text{ and } X \ge T_i \\
  0 & \text{if } X \notin R_i.
  \end{cases}
\end{equation*}
By this direct construction, the probability of outputting $a_i$ (respectively, $b_i$) precisely  matches the auxiliary probability $p_{a_i}$ (respectively, $p_{b_i}$) defined in \eqref{eq: p_ai} (respectively, \eqref{eq: p_bi}). 
Specifically, we have:
\begin{equation*}
  p_{a_i}  
  = \Pr(X \ge a_i) - \Pr(X \ge T_i)
  = \Pr(a_i \le X < T_i ) 
  = \Pr(X \in R_i \text{ and } X < T_i)
  = \Pr(\mathrm{SQ}_i(X) = a_i).
\end{equation*}
Likewise, analogous steps yield
\begin{equation*}
  p_{b_i} 
  = \Pr(X \in R_i \text{ and } X \ge T_i) 
  = \Pr(\mathrm{SQ}_i(X) = b_i).
\end{equation*}
These equivalences allow us to interpret our threshold queries as a form of binary stochastic quantization. 
A well-known property of a stochastic quantizer is that it ``rounds''~$X$ in a way that preserves its value in expectation. 
To formally show this, we evaluate the conditional expectation of $\text{SQ}_i(X)$ given $X=x$. Using the fact that $T_i \sim {\rm Unif}(a_i,b_i)$, we observe the following:
\begin{itemize}[topsep = 0pt, itemsep=0pt]
  \item If $x \notin R_i$, then $\mathrm{SQ}_i(x) = 0$, which trivially gives $\EE[\mathrm{SQ}_i(X) \mid X = x] = 0$.
  \item If $x \in R_i$, the conditional expectation simplifies to $\EE[\mathrm{SQ}_i(X) \mid X = x] = a_i \left( \frac{b_i - x}{b_i - a_i} \right) + b_i \left( \frac{x - a_i}{b_i - a_i} \right) = x$.
\end{itemize}
Using an indicator function, we can express this conditional expectation compactly for all $X$:
\begin{equation*}
  \label{eq: cond exp of SQ}
  \EE[\mathrm{SQ}_i(X) \mid X] = X \cdot \mathbf{1}(X \in R_i).
\end{equation*}
Combining the above findings and applying the law of total expectation yields the key identity:
\begin{equation*}
  \mu_i = 
  \EE[X \cdot \mathbf{1}(X \in R_i)] =
  \EE\left[ \EE \left[\mathrm{SQ}_i(X) \mid X \right] \right] =
  \EE[\mathrm{SQ}_i(X) ] =
  a_i \cdot p_{a_i} + b_i \cdot p_{b_i}.
\end{equation*}
It follows that to estimate the true local mean contribution $\mu_i$, it is sufficient to estimate $p_{a_i}$ and $p_{b_i}$. To do this, the learner forms unbiased estimates $\hat{p}_{a_i}$ and $\hat{p}_{b_i}$ using the empirical averages of randomized threshold queries. Specifically, the estimation procedure for $p_{a_i}$ operates as follows:
\begin{enumerate}[topsep=0pt, itemsep=0pt]
    \item Ask the agent $n_i$ threshold queries ``Is $X_{i,j} \ge a_i$?'' for $j = 1, \dots, n_i$, where $n_i$ is the regional sample budget to be determined in Step~5.
    \item Generate independent random variables $T_{i,j} \sim \mathrm{Unif}(a_i, b_i)$ for $j = n_i+1, \dots, 2n_i$.
    \item Ask the agent $n_i$ randomized threshold queries ``Is $X_{i,j} \ge T_{i,j}$?'' for $j = n_i+1, \dots, 2n_i$.
    \item Compute the empirical averages based on the 1-bit feedback and perform the subtraction according to~\eqref{eq: p_ai}.
\end{enumerate}
The learner forms the corresponding estimate $\hat{p}_{b_i}$ using an analogous procedure, utilizing $2n_i$ fresh samples with queries ``Is $X_{i,j} \ge T_{i,j}$?'' and ``Is $X_{i,j} \ge b_i$?'', where the random thresholds $T_{i,j} \sim {\rm Unif}(a_i,b_i)$ are freshly and independently sampled. We summarize the empirical estimates as
\begin{equation}
  \label{eq: estimates p_a}
  \hat{p}_{a_i} = \frac{1}{n_i} \left( \sum_{j=1}^{n_i} \mathbf{1}\left(X_{i,j} \ge a_i \right) \right) - \frac{1}{n_i} \left( \sum_{j=n_i+1}^{2n_i} \mathbf{1}\left(X_{i,j} \ge T_{i,j}\right) \right)
\end{equation}
and
\begin{equation}
  \label{eq: estimates p_b}
  \hat{p}_{b_i} = \frac{1}{n_i} \left( \sum_{j=2n_i+1}^{3n_i} \mathbf{1}\left(X_{i,j} \ge T_{i,j} \right) \right) - \frac{1}{n_i} \left( \sum_{j=3n_i+1}^{4n_i} \mathbf{1}\left(X_{i,j} \ge b_i \right) \right).
\end{equation}
This procedure consumes exactly $4n_i$ independent samples per region $R_i$. 
Because the region boundaries $a_i$ and $b_i$ are explicitly fixed after Step 3, the data collection for all empirical pairs $\{(\hat{p}_{a_i}, \hat{p}_{b_i})\}_{|i| \le i_{\max}}$ can be executed in a non-adaptive, parallel manner. Combining the empirical averages using $\hat{\mu}_i = a_i \hat{p}_{a_i} + b_i \hat{p}_{b_i}$ yields the final unbiased local estimate.

\textbf{Step 5 (Base Estimator and Sample Allocation)}: 
Let the base estimator $\hat{\mu}_{\text{base}} = \sum_{1 \le |i| \le i_{\max}} \hat{\mu}_i$ be the sum of the local estimates. To guarantee that the median-of-means wrapper in Step 6 succeeds, the base estimator $\hat{\mu}_{\text{base}}$ must achieve an estimation error of at most $\eps/2$ with a failure probability strictly less than $1/2$ (e.g., $\le 1/4$). 

Since $\hat{\mu}_{\text{base}}$ is an unbiased estimator of the clipped mean, i.e., $\EE[\hat{\mu}_{\text{base}}] = \sum_{1 \le |i| \le i_{\max}} \mu_i$, applying Chebyshev's inequality yields
\begin{equation}
\label{eq: chebyshev_base}
  \Pr\left(\left|\hat{\mu}_{\text{base}} - \sum_{1 \le |i| \le i_{\max}} \mu_i\right| \ge \frac{\eps}{2}\right) 
  = \Pr \left( \left| \hat{\mu}_{\text{base}} - \EE[\hat{\mu}_{\text{base}}] \right| \ge \frac{\eps}{2}  \right)
  \le \frac{\Var(\hat{\mu}_{\text{base}})}{(\eps/2)^2} 
  = \frac{4 \Var(\hat{\mu}_{\text{base}})}{\eps^2}.
\end{equation}
Therefore, it is sufficient to enforce the global variance constraint $\Var(\hat{\mu}_{\text{base}}) \le \eps^2 / 16$. 
Because the local estimates are constructed from independent random threshold queries, the variance of the base estimator decomposes as the sum of the local variances: 
$$\Var(\hat{\mu}_{\text{base}}) = \Var\left(\sum_{1 \le |i| \le i_{\max}} \hat{\mu}_i \right) = \sum_{1 \le |i| \le i_{\max}}  \Var(\hat{\mu}_i).$$
We analyze this in three substeps: 
5(a) bounding the local variance via tail probabilities, 
5(b) constructing sample allocation to satisfy the global variance constraint, 
and 5(c) evaluating the final sample complexity across the tail regimes.

\textit{Substep 5(a): Local Variance Bounds}.
By the symmetry of the partition construction ($R_{-i} = -R_i$), we can assume $i \ge 1$ (right-sided region) without loss of generality; an identical argument applies to the left-sided region.
Recall that $R_i = [a_i, b_i) = \sigma \cdot [m_{i-1} , m_i)$, where $m_0 = 0 $ and $m_i = 2^i$. Since $\hat{\mu}_i = a_i \cdot \hat{p}_{a_i} + b_i \cdot \hat{p}_{b_i}$ and $\hat p_{a_i}$ and $\hat p_{b_i}$ are formed using independent samples, the local variance decomposes as
\begin{equation*}
    \Var(\hat{\mu}_i) 
    = a_i^2 \cdot \Var(\hat{p}_{a_i} )  + b_i^2 \cdot \Var(\hat{p}_{b_i} )
    \le m_i^2 \cdot \sigma^2 \cdot \left( \Var(\hat{p}_{a_i} ) + \Var(\hat{p}_{b_i}) \right) .
\end{equation*}
We bound the variance of each probability estimate. Using the definition of~$\hat{p}_{a_i}$ (see~\eqref{eq: estimates p_a}) and the fact that the queries within each empirical average are i.i.d. (i.e., $X_{j} \sim X$ and $T_{ij} \sim T_i$), we have
\begin{align*}
    \Var(\hat{p}_{a_i} )
    &= \frac{1}{n_i} \cdot \Var(\mathbf{1}(X \ge a_i )) +
    \frac{1}{n_i} \cdot \Var(\mathbf{1}(X \ge T_i )) \\
    &\le \frac{1}{n_i} \cdot \mathrm{Pr}(X \ge a_i )  +
    \frac{1}{n_i} \cdot \mathrm{Pr}(X \ge T_i ) \\
    &\le  \frac{2}{n_i} \cdot \mathrm{Pr}(X \ge a_i ),
\end{align*}
where the first inequality follows from the Bernoulli variance property $\Var(\mathrm{Ber}(p)) = p(1-p) \le p$, and the second inequality follows because $T_i \ge a_i$.
An identical argument yields 
\begin{equation*}
    \Var(\hat{p}_{b_i}) 
    \le \frac{1}{n_i} \cdot \mathrm{Pr}(X \ge T_i )  +
        \frac{1}{n_i} \cdot \mathrm{Pr}(X \ge b_i ) 
    \le  \frac{2}{n_i} \cdot \mathrm{Pr}(X \ge a_i ). 
\end{equation*}
For the tail analysis, extend the right-sided regions by defining $R_j=\sigma[m_{j-1},m_j)$ for all $j \ge 1$, and write $p_j \coloneqq \Pr(X \in R_j)$. Using this, the tail probability can be expressed as $\Pr(X \ge a_i) = \sum_{j=i}^{\infty} p_j$. Substituting this into the local variance yields the following bound:
\begin{equation}\label{eq: local variance tail bound}
\Var(\hat{\mu}_i) \le \frac{4 \cdot m_i^2 \cdot \sigma^2}{n_i} \sum_{j=i}^{\infty} p_j.
\end{equation}

\textit{Substep 5(b): Sample Allocation and Global Variance Verification}. 
Summing over all $\imax$ regions of the right tail and swapping the order of summation (which is justified by the non-negativity of arguments and Tonelli's theorem) gives
\begin{equation}
\label{eq: global variance Fubini}
        \sum_{i=1}^{\imax} \Var(\hat{\mu}_i) 
    \le \sum_{i=1}^{\imax} \frac{4 \cdot m_i^2 \cdot \sigma^2}{n_i} \sum_{j=i}^{\infty} p_j = \sum_{j=1}^{\infty} p_j \cdot \left( \sum_{i=1}^{\min(j, \imax)} \frac{4 \cdot m_i^2 \cdot \sigma^2}{n_i} \right).
\end{equation}
Recall from Step 1 that the shifted mean is bounded by $|\mu| \le 4\sigma$.
Furthermore, the region boundaries satisfy $a_j > 4 \sigma$ if and only if $j \ge 4$.
Based on these observations, we split the outer sum of~\eqref{eq: global variance Fubini} into the two cases: (i) $j \le 3$ and (ii) $j \ge 4$.
For $j \le 3$, using the trivial bound $p_j \le 1$, their contribution to the global variance is at most
\begin{equation}
\label{eq: contribution from j<=3}
    \sum_{j=1}^{3} p_j \cdot \left( \sum_{i=1}^{\min(j, \imax)} \frac{4 \cdot m_i^2 \cdot \sigma^2}{n_i} \right)
    = O \left( \frac{\sigma^2 }{\min_{i \le 3} n_i}\right).
\end{equation}

We now consider $j \ge 4$.  If $X \in R_j$, then $X \ge m_{j-1} \sigma = 2^{j-1}\sigma$. Applying the triangle inequality and exploiting $|\mu| \le 4\sigma$ yields
\begin{equation*}
    X \in R_j \implies
        |X - \mu| 
    \ge X - |\mu| 
    \ge (2^{j-1} - 4) \cdot \sigma
    \ge 2^{j-2} \cdot \sigma
    = m_{j-2} \sigma
    \implies
    |X - \mu|^k \ge (m_{j-2} \sigma)^k.
\end{equation*}
Multiplying by the indicator $\mathbf{1}\{X \in R_j\}$ and taking the expectation across all $j \ge 4$ connects the region probabilities to the $k$-th central moment bound:
\begin{equation*}
    \sum_{j=4}^{\infty} p_j \cdot (m_{j-2}\sigma)^k 
    = \sum_{j=4}^{\infty} (m_{j-2}\sigma)^k \cdot  \EE\big[ \mathbf{1}\{X \in R_j\} \big] 
    \le \sum_{j=4}^{\infty} \EE[|X-\mu|^k \mathbf{1}\{X \in R_j\}] \le \EE[|X-\mu|^k] \le \sigma^k.
\end{equation*}
Dividing by $\sigma^k$ and noting that $m_{j-2}^k = 2^{k(j-2)} = 2^{jk}/2^{2k}$, this simplifies to
\begin{equation}
    \label{eq: tail bound m_jk}
     \sum_{j=4}^{\infty} p_j \cdot 2^{jk} \le 2^{2k} = O(1),
\end{equation}
where the $O(1)$ bound follows directly from Remark~\ref{rem: operative moment}, as we assume $k \le 3$ without loss of generality.  We propose setting the local sample allocation to scale as
\begin{equation*}
    n_i 
    = \Theta\left( \frac{\sigma^2}{\eps^2} \cdot 2^{i(2-k)} \right).
\end{equation*}
Under this allocation, the inner geometric sum of~\eqref{eq: global variance Fubini} is upper bounded by the following (extending the highest index from $\min(j,i_{\max})$ to $j$ for simplicity):
 \begin{equation}
 \label{eq: inner sum}
        \sum_{i=1}^{j} \frac{4 \cdot m_i^2 \cdot \sigma^2}{n_i}
     = O\left( \eps^2 \sum_{i=1}^j  2^{ik} \right)
     = O\left( \eps^2 \cdot 2^{jk} \right).
 \end{equation}
 Substituting this into the outer sum of~\eqref{eq: global variance Fubini} for $j \ge 4$ and combining with \eqref{eq: tail bound m_jk} (as well as \eqref{eq: contribution from j<=3}) yields:
 \[
   \sum_{i=1}^{i_{\max}} \Var(\hat{\mu}_i) 
    \le  O(\eps^2) + O \left( \eps^2 \sum_{j=4}^{\infty} p_j \cdot 2^{jk} \right)
    = O(\eps^2).
 \]
Applying the same argument to $Z = -X$ gives the identical bound for the left-sided regions, i.e.,
$\sum_{i=1}^{i_{\max}}\Var(\hat\mu_{-i})
  \le O(\eps^2)$.
Therefore,
\[
  \Var(\hat\mu_{\rm base})
  =
  \sum_{1\le |i|\le i_{\max}}\Var(\hat\mu_i)
  \le O(\eps^2).
\]
By scaling the sample allocation $n_i$ with a sufficiently large absolute constant, it follows that the base estimator satisfies the global variance constraint $\Var(\hat{\mu}_{\text{base}}) \le \eps^2/16$.

\textit{Substep 5(c): Total Sample Complexity}. 
The sample complexity for a single base estimator is given by the sum of the regional sample allocations:
\begin{equation*}
\sum_{1 \le |i| \le i_{\max}} 4n_i 
= 8 \sum_{i=1}^{i_{\max}} n_i
= \Theta\left( \frac{\sigma^2}{\eps^2} \cdot \underbrace{ \left( \sum_{i=1}^{i_{\max}} 2^{i(2-k)}  \right)}_{S_k}\right)
\end{equation*}
To establish explicit asymptotic bounds, we evaluate the geometric series $S_k$ across three distinct tail regimes:

\begin{enumerate}

    \item \textbf{Light-tailed distributions ($k > 2$)}:
  Because the exponent $2-k$ is strictly negative, $S_k$ is a convergent geometric series bounded by 
  \begin{equation*}
    S_k 
    =  \frac{2^{2-k}(1 - 2^{(2-k) \cdot i_{\max}})}{1 - 2^{2-k}}
    \le  \frac{2^{2-k}}{1 - 2^{2-k}}
    = \frac{1}{2^{k-2} - 1}.
  \end{equation*}
  As $k \to 2^+$, the denominator $2^{k-2} - 1 = \Theta\left( \ln 2 \cdot \left(k-2 \right)   \right) = \Theta(k-2)$ by a Taylor series expansion. 
  Therefore, $S_k = \Theta\big(\frac{1}{k-2} \big)$, and the sample complexity is bounded by 
  \begin{equation*}
    \sum_{1 \le |i| \le i_{\max}} n_i  
    = O\left(\frac{\sigma^2}{\eps^2} \cdot  \frac{1}{k-2} \right).
  \end{equation*}
  Note that if we drop the assumption $k \le 3$ from Remark \ref{rem: operative moment}, then this generalizes to $O\big(\frac{\sigma^2}{\eps^2} \cdot \max\big\{1,   \frac{1}{k-2}  \big\} \big)$, i.e., the factor $\frac{1}{k-2}$ is relevant as $k \to 2^+$ but not as $k \to \infty$.

  \item \textbf{Finite-variance distributions ($k = 2$)}:
  Because the exponent $2 - k$ is exactly zero, $S_k$ trivially evaluates to 
  \begin{equation*}
    S_k = \sum_{j=1}^{i_{\max}} 1 = \imax = \Theta\left(\log\frac{\sigma}{\eps}\right).
  \end{equation*} 
  The sample complexity is therefore bounded by
  \begin{equation*}
    \sum_{1 \le |i| \le i_{\max}} n_i  
    = O\left(\frac{\sigma^2}{\eps^2} \cdot  \log\left(\frac{\sigma}{\eps}\right) \right).
  \end{equation*}

    \item \textbf{Heavy-tailed distributions ($k \in (1, 2)$)}: 
  Because the exponent $2 - k$ is strictly positive, $S_k$ is a growing geometric series dominated by its final term:
  \begin{equation*}
    S_k 
    = \frac{2^{2-k}(2^{i_{\max}(2-k)} - 1)}{2^{2-k} - 1} 
    = \Theta\left( \frac{2^{i_{\max}(2-k)}}{2^{2-k} - 1} \right).
  \end{equation*}
  As $k \to 2^-$, the denominator $2^{2-k} - 1 = \Theta\left(\ln 2 \left( 2-k \right)  \right)  = \Theta(2-k)$ by a Taylor series expansion. 
  For the numerator, we recall from Steps 2 and 3 that $2^{\imax} \sigma = \Theta(t) = \Theta(\sigma \cdot (\sigma/ \eps)^{1/(k-1)})$, which yields
  \begin{equation*}
    2^{\imax (2-k)} 
    = (2^{\imax})^{2-k}
    = \Theta\left( \left(\frac{\sigma}{\eps} \right)^{\frac{2-k}{k-1}} \right).
  \end{equation*}
  Substituting these bounds into $S_k$ gives
  \begin{equation*}
    S_k
    = \Theta\left(\left(\frac{\sigma}{\eps} \right)^{\frac{2-k}{k-1}} \cdot \frac{1}{2-k} \right).
  \end{equation*}
  The sample complexity is therefore bounded by
  \begin{equation*}
    \sum_{1 \le |i| \le i_{\max}} n_i  
    = O\left( \frac{\sigma^2}{\eps^2} \cdot  \left(\frac{\sigma}{\eps}\right)^{\frac{2-k}{k-1}} \cdot  \frac{1}{2-k}  \right) 
    = O\left( \left(\frac{\sigma}{\eps}\right)^{\frac{k}{k-1}} \cdot  \frac{1}{2-k}  \right).
  \end{equation*}

\end{enumerate}

\textbf{Step 6 (Median-of-Means):}  While the base estimator $\hat{\mu}_{\text{base}}$ satisfies the target global variance constraint, it achieves $\eps$-accuracy with only constant probability. 
We boost this success probability using the median-of-means framework. 
The algorithm repeats the base estimation independently $K = \lceil 8\log(1/\delta_{\rm ref}) \rceil$ times to obtain $\hat{\mu}_{\text{base}}^{(1)}, \dots, \hat{\mu}_{\text{base}}^{(K)}$ 
and takes their median as the final estimate:
\[
\hat{\mu} = \mathrm{median}\Big( \hat{\mu}_{\text{base}}^{(1)}, \dots, \hat{\mu}_{\text{base}}^{(K)}\Big).
\]
For $r =1, \dotsc, K$, let $Y_r$ be the indicator variable for the failure event of the $r$-th batch:
\[
Y_r =
\mathbf{1}\left(
  \left|\hat{\mu}_{\text{base}}^{(r)}
  - \sum_{1 \le |i| \le \imax} \mu_i  \right| > \frac{\eps}{2}
\right),
\]
and let $S = \sum_{r=1}^K Y_r$ denote the total number of failures. 
By Chebyshev's inequality (as shown in~\eqref{eq: chebyshev_base}) and our variance bound $\Var(\hat{\mu}_{\text{base}}) \le \eps^2 / 16$ established in Step~5, the variables $Y_1, \dots, Y_K$ are i.i.d. Bernoulli random variables with $\Pr(Y_r = 1) \le 1/4$. This implies $\EE[S] \le K/4$.

If the final median estimate $\hat{\mu}$ deviates from the clipped mean by more than $\eps/2$, it must be the case that at least half of the individual base estimates failed (i.e., $S \ge K/2$). 
Therefore, applying Hoeffding's inequality yields
\[
\Pr \left( \left| \hat{\mu} - \sum_{1 \le |i| \le \imax} \mu_i  \right| > \frac{\eps}{2} \right)
\le
\Pr \left(  S  \ge \frac{K}{2} \right)
\le
\Pr \left(   S - \EE[S] \ge \frac{K}{4} \right)
\le
\exp\left( - \frac{K}{8}\right) \le \delta_{\rm ref}.
\]

\textbf{Putting it All Together.}
Let $\mathcal E_{\rm loc}$ be the localization success event from Step~1, which satisfies
$\Pr(\mathcal E_{\rm loc}^c)\le \delta_{\rm loc}.$
Conditioned on $\mathcal E_{\rm loc}$, the recentering step guarantees
$|\mu|\le 4\sigma$, and the cutoff choice in Step~2 gives the deterministic
truncation bound
\[
  \left\lvert
    \mu - \sum_{1 \le |i| \le \imax} \mu_i  
  \right\rvert
  \le
  \frac{\eps}{2}.
\]
Denote the event
\[
  \mathcal E_{\rm ref}
  \coloneqq
  \left\{
    \left|
      \hat{\mu} - \sum_{1 \le |i| \le \imax} \mu_i  
    \right|
    \le
    \frac{\eps}{2}
  \right\},
\]
which satisfies $\Pr(\mathcal E_{\rm ref}^c\mid \mathcal E_{\rm loc}) \le
  \delta_{\rm ref}$
by Step~6.
Therefore, on the event $\mathcal E_{\rm loc}\cap\mathcal E_{\rm ref}$,
the triangle inequality gives $|\hat{\mu} - \mu| \le \eps$.
Finally,
\[
\Pr(|\hat{\mu}-\mu|>\eps)
\le
\Pr(\mathcal E_{\rm loc}^c)
+
\Pr(\mathcal E_{\rm ref}^c\cap \mathcal E_{\rm loc})
\le
\delta_{\rm loc}
+
\delta_{\rm ref}
\le
\delta.
\]
For Theorem~\ref{thm: main}, we set
$\delta_{\rm loc}=\delta_{\rm ref}=\delta/2$, giving the claimed
$(\eps,\delta)$-PAC guarantee.

\section{Lower Bounds and Adaptivity Gap}\label{appendix: Lower Bounds}
\subsection{Proof of Theorem~\ref{thm: adaptive lower bound} (Matching Lower Bound)}\label{appendix: Lower Bound for Adaptive Queries}

Theorem~\ref{thm: adaptive lower bound} establishes a lower bound that decomposes into two components: a tail-dependent base complexity, and a tail-independent additive localization cost. We prove these components separately.

\subsubsection{\underline{Part 1: Tail-Dependent Base Complexities}}

In the absence of communication constraints, the minimax sample complexity bounds for unquantized mean estimation are well known:
\[
    n = 
    \begin{cases}
    \Omega \left(  \dfrac{\sigma^2}{\eps^2}    \cdot
    \log\left(\dfrac{1}{\delta}\right) \right) & \text{if } k \ge 2
    \\ \\
    \Omega \left(  \left( \dfrac{\sigma}{\eps}  \right)^{\frac{k}{k-1}}  \cdot
    \log \left(\dfrac{1}{\delta}\right) \right)  & \text{if } k \in (1, 2).
    \end{cases}
\]
For instance, these can be derived via a reduction to distinguishing two Bernoulli distributions for $k \ge 2$~\cite[Section 4]{LeeCSCI1951}, and two scaled Bernoulli distributions for $k \in (1, 2)$~\cite[Section 4.3]{devroye2016sub}. 

However, these unquantized lower bounds are insufficient to verify the strict optimality of our estimator for the finite-variance case ($k=2$), as our upper bound contains an additional $O(\log(\sigma/\eps))$ factor. We now prove that this extra logarithmic penalty is not an artifact of our analysis, but a fundamental information-theoretic bottleneck imposed by 1-bit quantization. 

Specifically, we prove that any (potentially adaptive) 1-bit mean estimator that is $(\eps,\delta)$-PAC for the family $\mathcal{D}(2,\lambda,\sigma)$ with $\eps \le c\sigma$ (for a sufficiently small constant $c > 0$) and $\lambda \ge 3\eps$ must satisfy:
\[
n = \Omega\left(\frac{\sigma^2}{\eps^2} \cdot \log\left(\frac{\sigma}{\eps}\right) \cdot \log\left(\frac{1}{ \delta} \right) \right).
\]

\begin{proof}[Proof of the base complexity for $k=2$]
We will employ a result from \cite[Theorem 2.2(iii)]{tsybakov2009introduction} that combines Le Cam's two-point method combined with the Bretagnolle–Huber inequality. We construct a null distribution $D_0$ with mean $0$ and an alternative distribution $\bar D$ with mean $3\eps$. Both belong to $\mathcal{D}(2,\lambda,\sigma)$ when $\sigma = \Omega(\eps)$ with a sufficiently large hidden constant. Distinguishing them with error $\delta$ requires the stated number of samples.

\textbf{Step 1 (Constructions of the Distributions):} Set
\[
M = \left\lfloor \frac{1}{2} \log_2\left(\frac{\sigma}{3\eps}\right) \right\rfloor,
\]
so that $2^M \le \sqrt{\sigma/(3\eps)}$ and $M = \Theta(\log(\sigma/\eps))$. Define the grid points $x_i = 2^i \cdot \sigma$ for $i=1,\dots,M$.

\emph{Null distribution $D_0$:} 
Place symmetric point masses at $\pm x_i$ with probabilities
\[
q_i = \frac{1}{2M \cdot 2^{2i}},\qquad i=1,\dots,M.
\]
The remaining mass
\[
    1-\sum_{i=1}^M 2q_i 
    = 1 - \frac{1}{M} \left( \sum_{i=1}^M 4^{-i} \right)
    \ge 1 - \frac{1}{3M}
    > \frac{1}{2}
\]
is placed at $0$. By symmetry, we have $\EE_{D_0}[X]=0$, and
\[
\operatorname{Var}_{D_0}(X)=\sum_{i=1}^M 2q_i x_i^2
= \sum_{i=1}^M 2\left(\frac{1}{2M2^{2i}}\right)2^{2i}\sigma^2
= \sum_{i=1}^M \frac{\sigma^2}{M} = \sigma^2,
\]
which implies $D_0\in\mathcal{D}(2,\lambda,\sigma)$.

\emph{Alternative distribution $\bar D$:} 
For each $j=1,\dots,M$, construct $D_j$ by taking $D_0$ and shifting mass $p_j$ from $-x_j$ to $+x_j$, where
\[
p_j = \frac{3\eps}{2^{j+1} \cdot \sigma}.
\]
For this to be valid, we require $p_j\le q_j$, which follows from the ratio
\[
\frac{p_j}{q_j} = \frac{3\eps}{2^{j+1} \cdot \sigma} \cdot 2M \cdot 2^{2j}
= \frac{3M \cdot \eps \cdot 2^j}{\sigma} 
\le \frac{3M \cdot \eps \cdot 2^M}{\sigma} 
\le \frac{3M \cdot \eps}{\sigma} \cdot  \sqrt{\frac{\sigma}{3\eps}}
= \left\lfloor \frac{1}{2} \log_2\left(\frac{\sigma}{3\eps}\right) \right\rfloor \cdot \sqrt{\frac{3\eps}{\sigma}}
\]
and the fact that $\eps \le c\sigma$ for a sufficiently small absolute constant $c$ (e.g., $c \le 0.01$), noting also that $\lim_{z \to 0} \log(1/z)\sqrt{z} = 0$.  We then define the uniform mixture $\bar D = \frac{1}{M}\sum_{j=1}^M D_j$. For each constituent distribution $D_j$, the mean is
\[
\EE_{D_j}[X] 
= 2p_j \cdot x_j 
= \frac{2 \cdot 3\eps \cdot 2^j \cdot \sigma}{2^{j+1} \cdot \sigma}  = 3\eps.
\]
Because the mean of $\bar D$ is an average of the constituent means, it follows that $\EE_{\bar D}[X] = 3\eps$. Furthermore, since mass was simply shifted from $-x_j$ to $x_j$, the second moment is preserved, ensuring
\[
\mathrm{Var}_{\bar D}(X) \le \EE_{\bar D}[X^2] 
= \EE_{D_0}[X^2] = \mathrm{Var}_{D_0}(X) = \sigma^2.
\]
Thus, we have $\bar D \in \mathcal{D}(2,\lambda,\sigma)$; note that the assumption $\lambda \ge 3\eps$ in Theorem \ref{thm: adaptive lower bound} implies that the mean $\EE_{\bar D}[X] = 3\eps$ indeed lies in $[-\lambda,\lambda]$.

\textbf{Step 2 (Reduction to Binary Hypothesis Testing):}
An $(\eps,\delta)$-PAC estimator must output an estimate $\hat\mu\in[-\eps,\eps]$ with probability $\ge1-\delta$ under $D_0$, and $\hat\mu\in[2\eps,4\eps]$ under $\bar D$ with probability $\ge1-\delta$. Since these two target intervals are strictly disjoint, any valid $(\eps,\delta)$-PAC estimator can be used to form a binary decision rule that distinguishes between $D_0$ and $\bar D$ with an error probability of at most $\delta$.

\textbf{Step 3 (Per‑query KL divergence bound):}
Fix an arbitrary 1‑bit query and let $S\subset\mathbb{R}$ be the corresponding measurable set. Define $P_0(S)=\Pr_{D_0}(X\in S)$ and $\bar P(S)=\Pr_{\bar D}(X\in S)$. 

We now upper bound the KL divergence $D_{\mathrm{KL}}\left(\bar P(S) \,\middle\|\, P_0(S) \right)$. Because the KL divergence between two Bernoulli distributions is invariant to swapping the labels of the outcomes (i.e., replacing $S$ with $S^c$), we have:
\[
D_{\mathrm{KL}}\left(\bar P(S^c) \,\middle\|\, P_0(S^c) \right) = 
D_{\mathrm{KL}}\left(\bar P(S) \,\middle\|\, P_0(S) \right).
\]
Consequently, we may assume without loss of generality that $P_0(S)\le 1/2$. Then, because $D_0$ has strictly more than $1/2$ of its mass at $0$, the set $S$ cannot contain $0$; thus $S$ can only capture mass from the grid points $\{\pm x_i\}$.

Consider $j^* = \min\{j\ge 1 : S\cap\{x_j,-x_j\}\neq\emptyset\}$. If no such index exists, then $P_0(S)=\bar P(S)=0$ and the KL divergence is exactly $0$. Otherwise,
\begin{equation}\label{eq: P_0(S) bound}
P_0(S) \ge q_{j^*} = \frac{1}{2M \cdot 2^{2j^*}}.
\end{equation}
The difference in probabilities is bounded by the total shifted mass from indices $j^*$ onwards:
\begin{equation}\label{eq: |barP-P_0| bound}
\left|\bar P(S)-P_0(S) \right|
= \left|\frac{1}{M}\sum_{j=j^*}^M \bigl(P_{D_j}(S)-P_0(S)\bigr)\right|
\le \frac{1}{M}\sum_{j=j^*}^M p_j
= \frac{1}{M}\sum_{j=j^*}^M \frac{3\eps}{2^{j+1} \cdot \sigma}
\le \frac{3\eps}{M\sigma}\cdot\frac{1}{2^{j^*}}.
\end{equation}
Using~\eqref{eq: P_0(S) bound}--\eqref{eq: |barP-P_0| bound}, alongside the standard inequality $D_{\mathrm{KL}}( \mathrm{Bern}(p) \,\|\, \mathrm{Bern}(q) ) \le \frac{(p-q)^2}{q(1-q)}$ and the condition $P_0(S) \le 1/2$, we obtain:
\[
D_{\mathrm{KL}}(\bar P(S) \,\|\, P_0(S))
\le \frac{(\bar P(S)-P_0(S))^2}{P_0(S) \cdot (1-P_0(S))}
\le \frac{2(\bar P(S)-P_0(S))^2}{P_0(S)}
\le 2\cdot \frac{\left(\frac{3\eps}{M\sigma2^{j^*}}\right)^2}{\frac{1}{2M2^{2j^*}}}
= \frac{36\eps^2}{M\sigma^2}.
\]

\textbf{Step 4 (Adaptive Protocol and Chain Rule):}
While the estimator's sequential querying strategy may be randomized, Yao's minimax principle states that the worst-case error probability of any randomized algorithm over $\{D_0, \bar{D}\}$ is lower-bounded by the average error probability of the optimal deterministic algorithm under a uniform prior over $\{D_0, \bar{D}\}$.  Therefore, to establish our lower bound, we may assume without loss of generality that the algorithm is deterministic (in the choice of what query to make at each time step given the information received so far, 
and in the procedure for forming the final estimate).

Under a deterministic algorithm, each measurable query set $S_t$ is a fixed function of the past 1-bit responses $Y^{t-1} = (Y_1, \dots, Y_{t-1})$. Denote by $P_{0, n}$ and $\bar{P}_{n}$ the joint distributions of the $n$-length response transcript under $D_0$ and $\bar D$, respectively. By the chain rule for KL divergence (see \cite[Theorem 2.16(c)]{Polyanskiy_Wu_2025}), we have:
\[
D_{\mathrm{KL}}\left(\bar{P}_{n}  \,\middle\|\,  P_{0, n} \right)
= \sum_{t=1}^n \EE_{\bar{P}_{Y^{t-1}}}\Bigl[ D_{\mathrm{KL}}\bigl( \bar{P}_{Y_t|Y^{t-1}} \,\big\|\, P_{0, Y_t|Y^{t-1}} \bigr) \Bigr].
\]
Conditioned on a specific realization of the past responses $Y^{t-1}$, the query $S_t$ is fixed. Thus, the conditional distributions $\bar{P}_{Y_t|Y^{t-1}}$ and $P_{0, Y_t|Y^{t-1}}$ are Bernoulli distributions induced by evaluating the static set $S_t$ under $\bar{D}$ and $D_0$. Applying the universal pointwise bound from Step~3 to each conditional term yields the total transcript bound:
\begin{equation}\label{eq: N-step KL divergence bound}
D_{\mathrm{KL}}\left(\bar{P}_n \,\middle\|\, P_{0, n} \right) \le n \cdot \frac{36\eps^2}{M\sigma^2}.
\end{equation}

\textbf{Step 5 (Lower bound via Bretagnolle–Huber):}
By applying the Bretagnolle–Huber inequality to Le Cam's method, the result in \cite[Theorem 2.2(iii)]{tsybakov2009introduction} states that the average error probability $\delta$ of distinguishing the two hypotheses under a uniform prior is lower bounded as follows:
\[
\delta \ge \frac{1}{4} \exp\left(- D_{\mathrm{KL}}(\bar{P}_n \,\|\, P_{0, n}) \right)
\implies
D_{\mathrm{KL}}(\bar{P}_n \,\|\, P_{0, n}) \ge \log\left(\frac{1}{4 \delta} \right).
\]
Applying this to~\eqref{eq: N-step KL divergence bound} and rearranging, we obtain the required sample complexity:
\[
n \ge \frac{M\sigma^2}{36\eps^2} \cdot \log\left(\frac{1}{4 \delta} \right)
= \Omega\left(\frac{\sigma^2}{\eps^2} \cdot \log \left(\frac{\sigma}{\eps} \right) \cdot \log\left(\frac{1}{\delta} \right) \right),
\]
which completes the proof.
\end{proof}

\subsubsection{\underline{Part 2: Localization Cost}}
It remains to establish the additive $n = \Omega\left(\log \frac{\lambda}{\sigma}\right)$ localization cost for all tail regimes $k > 1$.

We create $N = \Theta(\lambda/\sigma)$ instances of ``hard-to-distinguish'' distribution pairs, which we will reuse in the proof of Theorem~\ref{thm: non-adaptive lower bound} in Appendix~\ref{appendix: Lower Bound for Non-adaptive Queries}.
Divide $[-\lambda, \lambda]$ into a grid of
$N  = \lambda / \sigma - 1 $
``center-points'' spaced $2\sigma$
apart,\footnote{For convenience, we assume
that $\lambda$ is an integer multiple of
$2\sigma$.  This is justified by a simple
rounding argument and the fact that when
$\lambda = \Theta(\sigma)$ the
$\Omega\big(\log\frac{\lambda}{\sigma}
\big)$ lower bound is trivial.  } i.e., the
center-points are
\begin{equation}
\label{eq: c_j spacing}
c_j = - \lambda + 2j \sigma \quad
\text{for each }j = 1, 2 \dots, N.
\end{equation}
For each instance $j$, we define two
probability distributions
$D_{j, -}$ and $D_{j, +}$, each with a
two-point support set $\{c_j-\sigma/2,
c_j+\sigma/2\}$, as follows:
\begin{equation}
\label{eq: D_j+ and D_j-}
\begin{aligned}
  D_{j,-} \colon& \Pr\left(X = c_j +
  \frac{\sigma}{2} \right) = \frac{1}{2}
  - \frac{\eps}{\sigma} =
  1 - \Pr\left(X = c_j - \frac{\sigma}{2} \right)
  \implies
  \EE[X ] = c_j - \eps \\
  D_{j,+} \colon & \Pr\left(X = c_j +
  \frac{\sigma}{2} \right) = \frac{1}{2}
  + \frac{\eps}{\sigma} =
  1 - \Pr\left(X = c_j - \frac{\sigma}{2} \right)
  \implies
  \EE[X ] = c_j + \eps.
\end{aligned}
\end{equation}
By construction, these ``hard'' distributions satisfy the structural properties required to be members of our target distribution families:
\begin{itemize}
\item \textbf{Bounded Mean:} By the assumption $\eps < \sigma/2$, the mean of each of these $2N$ distributions is contained within the search range $[-\lambda, \lambda]$.

\item \textbf{Bounded Support and Sub-Gaussianity:} The support of each distribution is bounded to an interval of exactly length $\sigma$ (the distance between $c_j - \sigma/2$ and $c_j + \sigma/2$). By Hoeffding's lemma, any random variable bounded in an interval of length $\sigma$ is sub-Gaussian with a variance proxy of at most $\sigma^2/4 \le \sigma^2$.

\item \textbf{Universal Moment Bounds:} For any of these distributions, the maximum deviation of a sample from its true mean $\mu$ is $|X - \mu| \le |(c_j \pm \sigma/2) - (c_j \mp \eps) | \le \sigma/2 + \eps$. Since $\eps < \sigma/2$, we are guaranteed that $|X - \mu| < \sigma$. Consequently, the $k$-th central moment satisfies $\EE[|X - \mu|^k] \le \sigma^k$ for all $k > 1$. 
\end{itemize}
Thus, this specific hard subset of discrete, bounded distributions belongs to the family $\mathcal{D}(k, \lambda, \sigma)$ across all tail regimes studied in this paper (i.e, all $k > 1$), ensuring our lower bound is applicable in all such regimes.

By the above construction, when the distributions are
restricted to only these $2N$
distributions, the task of being able to
form an $\eps$-good estimation of the
true mean of each unknown underlying
distribution is at least as hard as being
able to distinguish the distributions from
each other.\footnote{Strictly speaking this
is true when the algorithm is required to
attain accuracy \emph{strictly smaller}
than $\eps$, rather than {\em smaller
or equal}, but this distinction clearly has
no impact on the final result stated using
$O(\cdot)$ notation, and by ignoring it we
can avoid cumbersome notation.}  We proceed
to establish a lower bound for this goal of
\emph{identification}, also known as
\emph{multiple hypothesis testing}.

Let $\Theta$ be a uniform random variable
over the $2N$ distributions, whose entropy is given by
\begin{equation}
\label{eq: entropy of uniform dist}
H(\Theta) = \log(2N).
\end{equation}
Fix an adaptive mean estimator
that makes $n$ queries, and let $Y^n =
(Y_1, \dots, Y_n)$ be the resulting binary responses.
Using the chain rule for mutual information
(see e.g.~\cite[Theorem
3.7]{Polyanskiy_Wu_2025}) and the fact that
each query yields at most 1 bit of
information, we have
\begin{equation}
\label{eq: trivial mutual info bound}
I(\Theta; Y^n)
= \sum_{r=1}^n   I\big(\Theta; Y_r \mid Y^{r-1}\big)
\le \sum_{r=1}^n  H \big(Y_r \mid Y^{r-1} \big)
\le \sum_{r=1}^n  H (Y_r)
\le \sum_{r=1}^n  1
= n.
\end{equation}
Moreover, Fano’s inequality
(see~\cite[Theorem 3.12]{Polyanskiy_Wu_2025}) gives:
\begin{equation}
\label{eq: Fano inequality}
H(\Theta \mid Y^n) \leq H_2(\delta) +
\delta \log (2N-1)
\le 1 + \delta \log(2N),
\end{equation}
where $\delta$ is the error probability
and $H_2(p) = -p \log p - (1-p) \log(1-p)$
is the binary entropy function.
Using~\eqref{eq: entropy of uniform
dist}--\eqref{eq: Fano inequality} and the
definition of mutual information, we obtain
\begin{equation}
n
\ge I(\Theta; Y^n)
= H(\Theta) - H(\Theta \mid Y^n)
\ge  \log(2N) - 1 - \delta \log (2N)
= (1-\delta) \log(2N) - 1.
\end{equation}
Combining this with $N =
\Theta(\lambda/\sigma)$, we have
\[
n
= \Omega( \left(1-\delta \right) \log N )
= \Omega\left(\log \frac{\lambda}{\sigma}\right)
\]
as desired.

\subsection{Proof of Theorem~\ref{thm: non-adaptive lower bound} (Adaptivity Gap)}
\label{appendix: Lower Bound for Non-adaptive Queries}

We consider the same instance as that of
Section \ref{appendix: Lower Bound for Adaptive Queries}, and accordingly re-use
the notation therein.
Before proving Theorem~\ref{thm:
non-adaptive lower bound}, we first
introduce the idea of an interval query
being ``informative'' or ``uninformative''
for distinguishing between the
distributions $D_{j,-}$ and $D_{j,+}$.
\begin{definition}[Informative Interval
Queries] \label{def:informative}
For a fixed interval query $Q =
``\text{Is } X \in [\alpha_t, \beta_t]?"$, we say that
$Q$ is informative for the $j$-th pair of
distributions $(D_{j,-},D_{j,+})$ if its
binary feedback $B = \mathbf{1}\left\{X
\in [\alpha_t, \beta_t] \right\}$ satisfies
\[
  \Pr_{X \sim  D_{j,-}}(B = 1) \ne \Pr_{X
  \sim  D_{j,+}}(B =1).
\]
Otherwise, $Q$ is said to be uninformative.
\end{definition}
The following lemma shows that each
interval query can be simultaneously
informative for at most two different pairs.
\begin{lemma}
\label{lem: informative for at most 2 pairs}
An interval query $Q =  ``\text{Is } X
\in [\alpha_t, \beta_t]?"$can be simultaneously
informative for at most two different
$(D_{j,-},D_{j,+})$ pairs, i.e., at most
two different values of $j$.
\end{lemma}
\begin{proof}[Proof of Lemma~\ref{lem:
informative for at most 2 pairs}]
The claim follows from the following two facts:
\begin{enumerate}
  \item For a fixed distribution pair
    (indexed by $j$), an interval query
    $Q = $ ``$\text{Is } X \in [\alpha_t, \beta_t]?$''
    is informative for distinguishing
    between $D_{j,-}$ and $D_{j,+}$ only
    if $[\alpha_t, \beta_t]$ contains exactly one of
    the two support points $\{c_j \pm
    \sigma/2 \}$, i.e., $\big| [\alpha_t, \beta_t]
    \cap \{c_j \pm \sigma/2 \} \big| = 1$.

  \item
    There are at most two indices $j$ for
    which $\big| [\alpha_t, \beta_t] \cap \{c_j \pm
    \sigma/2 \} \big| = 1$.
\end{enumerate}
Fact 1 can be verified by analyzing the
binary feedback $B = \mathbf{1}\left\{X
\in [\alpha_t, \beta_t] \right\}$ for all cases of
$[\alpha_t, \beta_t] \cap \{c_j \pm \sigma/2 \}$:
\begin{equation*}
  \big|
  \left[\alpha_t, \beta_t \right] \cap \{c_j \pm
  \sigma/2\} \big|  \in\{0,2\}
  \implies
  \Pr_{X \sim  D_{j,-}}(B = 1) = \Pr_{X
  \sim  D_{j,+}}(B =1 )
  \implies
  Q \text{ is uninformative},
\end{equation*}
and
\begin{equation}
  \label{eq: bernoulli p+ and p-}
  \big|
  \left[\alpha_t, \beta_t \right] \cap \{c_j \pm
  \sigma/2\} \big|  = 1
  \implies
  \bigg|
  \Pr_{X \sim  D_{j,-}}(B = 1) - \Pr_{X
  \sim  D_{j,+}}(B =1 )
  \bigg| = \frac{2 \eps}{\sigma}
  \implies
  Q \text{ is informative}.
\end{equation}
For Fact 2, we first observe
from~\eqref{eq: c_j spacing} that the
support points of all $2N$ distributions satisfy
\[
  c_1- \frac{\sigma}{2} < c_1 +
  \frac{\sigma}{2} < c_{2} - \frac{\sigma}{2} <
  \cdots <  c_{N} - \frac{\sigma}{2} <
  c_{N} + \frac{\sigma}{2},
\]
with each pair $j$ having a unique
disjoint interval $(c_j - \sigma/2, c_j + \sigma/2)$
between its support points. An interval
$[\alpha_t, \beta_t]$ satisfies $\big| [\alpha_t, \beta_t] \cap
\{c_j \pm \sigma/2 \} \big| = 1$ if and
only if exactly one endpoint of $[\alpha_t, \beta_t]$
lies in the interval $(c_j - \sigma/2,
c_j + \sigma/2)$. Since the gaps are
disjoint and $[\alpha_t, \beta_t]$ has only two
endpoints, it follows that at most two
indices $j$ satisfy $\big| [\alpha_t, \beta_t] \cap
\{c_j \pm \sigma/2 \} \big| = 1$.
\end{proof}

\begin{proof}[Proof of Theorem~\ref{thm:
non-adaptive lower bound}]

Consider an arbitrary algorithm that
makes $n$ non-adaptive interval queries.
Recall the set of $2N$ distributions $\{
D_{j,-}, D_{j,+} \}_{j=1}^{N} \subseteq
\mathcal{D(\lambda, \sigma)}$ constructed
in the proof of Theorem~\ref{thm:
adaptive lower bound}, where $N =
\lambda/ \sigma -1$.  We will again
establish a lower bound for this ``hard
subset'' of distributions, but with
different details to exploit the
assumption of non-adaptive interval queries.

Recall from Section~\ref{appendix: Lower
Bound for Adaptive Queries} that the
means of the $2N$ distributions are
pairwise separated by $2\eps$ or
more, and thus, attaining
$\eps$-accuracy implies being able to
identify the underlying distribution from
the hard subset.  We proceed to establish
a lower bound for this goal of
identification (multiple hypothesis testing).

Suppose that the true distribution is
drawn uniformly at random from the $2N$
distributions in the hard subset.  As we argued in Section~\ref{appendix: Lower Bound for Adaptive Queries}, by
Yao’s minimax principle, the worst-case
error probability is lower bounded by the
average-case error probability of the
best \emph{deterministic} strategy, so we
may assume that the algorithm is
deterministic.

Letting $(\hat{j},\hat{s})$ be the
estimated index (in $\{1,\dotsc,N\}$) and
sign (in $\{1,-1\}$), the average-case
error probability is given by
\begin{align}
  \Pr({\rm error})
  &= \frac{1}{2N} \sum_{j=1}^N \sum_{s
  \in \{+1,-1\}} \Pr\nolimits_{j,s}(
  (\hat{j},\hat{s}) \ne (j,s)  ) \\
  &\ge \frac{1}{N}\sum_{j=1}^N \bigg(
    \underbrace{\frac{1}{2} \Pr\nolimits_{j,+}\big(
      \hat{s} \ne 1\big) +
      \frac{1}{2}\Pr\nolimits_{j,-}\big( \hat{s} \ne
    -1 \big)}_{=: \Pr\nolimits_j({\rm error})}
  \bigg), \label{eq:def_Pr_j}
\end{align}
where $\Pr_{j,s}$ denotes probability
when the underlying distribution is $D_{j,s}$.

For each $j = 1, \dots, N$, we define
$n_j$ to be the algorithm's total number
of interval queries that are informative
(in the sense of Definition
\ref{def:informative}) for distinguishing
between $D_{j,-}$ and $D_{j,+}$.  Since
the algorithm is deterministic and the
$n$ queries are assumed to be
non-adaptive (i.e., they must all be
chosen in advance), it follows that the
values $\{n_j\}_{j=1}^N$ are also deterministic.

Recall from~\eqref{eq: bernoulli p+ and
p-} that each informative query provides
binary feedback that follows either
$\mathrm{Bern}(p_{+})$ or
$\mathrm{Bern}(p_{-})$, where $p_{+} =
1/2 + \eps/\sigma$ and $p_{-} =1/2  -
\eps/\sigma = 1 - p_{+}$.
Distinguishing between these two cases is
a \emph{binary hypothesis testing}
problem, and the associated error
probability $\Pr_j(\text{error})$ is
given by the $j$-th summand in~\eqref{eq:def_Pr_j}.

Using standard binary hypothesis testing
lower bounds~\cite[Theorem
11.9]{LeeCSCI1951}, we have\footnote{We
  have re-arranged their result to express
other quantities in term of the error probability.}
\begin{equation}
  \label{eq: n_j hypothesis testing lower
  bound reexpressed}
  \Pr\nolimits_j(\text{error}) >
  \exp\left( -c' \cdot n_j \cdot
  d_H^2(p_{+}, p_{-})  \right)
\end{equation}
for some constant $c'$, where
$d_H^2(\mathbf{p}, \mathbf{q}) =
\frac{1}{2}\sum_{i} \left(\sqrt{p_i} -
\sqrt{q_i} \right)^2$
is the squared Hellinger distance.
For $\mathrm{Bern}(p_{+})$ and
$\mathrm{Bern}(p_{-})$, we have the
following standard calculation:
\begin{equation}
  \label{eq: d_H calculation bound}
  d_H^2(p_{+}, p_{-}) =
  \left(\sqrt{p_+} - \sqrt{p_-} \right)^2 =
  \left( \frac{p_+  - p_-}{\sqrt{p_+} +
  \sqrt{p_-} }\right)^2 =
  \frac{ |p_+  - p_-|^2}{  \left(  1 +
  2\sqrt{p_+  p_-} \right)^2 } =
  \Theta\left( |p_{+} - p_{-}|^2\right) =
  \Theta\left(\frac{\eps^2}{\sigma^2} \right),
\end{equation}
where the equalities follow from $p_+ + p_- = 1$ and $p_+  p_-
\in [0, 1/4]$.
Combining~\eqref{eq: n_j hypothesis
testing lower bound reexpressed}
and~\eqref{eq: d_H calculation bound}, we obtain
\begin{equation}
  \Pr\nolimits_j(\text{error})
  >
  \exp\left( -c'' \cdot \frac{n_j \,
  \eps^2}{\sigma^2}  \right)
\end{equation}
for some constant $c'' > 0$.
Applying Jensen's inequality (since
$\exp$ is convex) and using
$\sum_{j=1}^N n_j \le 2n$ (see
  Lemma~\ref{lem: informative for at most 2
pairs}), it follows that
\[
  \frac{1}{N} \sum_{j=1}^N \Pr\nolimits_j(\text{error}) >
  \frac{1}{N} \sum_{j=1}^N \exp\left(
  -c'' \cdot \frac{n_j \, \eps^2}{\sigma^2}  \right)
  \ge
  \exp\left( -c'' \cdot
    \frac{\eps^2}{\sigma^2} \cdot
  \frac{1}{N} \sum_{j=1}^N n_j \right)
  \ge
  \exp\left( -c'' \cdot
    \frac{\eps^2}{\sigma^2} \cdot
  \frac{2n}{N}  \right).
\]
To complete the proof, suppose that
\[
  n
  <
  \frac{1}{4c''} \cdot \frac{\lambda
  \sigma}{ \eps^2} \log\left(\frac{1}{\delta}\right).
\]
Using 
\[
    1 \le \frac{\lambda}{\sigma} = N + 1 \le 2N
    \iff
    \frac{1}{N} \le \frac{2 \sigma}{\lambda}
\]
we have
\[
    \frac{\eps^2}{\sigma^2} \cdot
  \frac{2c''n}{N} 
  \le 
  \frac{\eps^2}{\sigma} \cdot
  \frac{4c''n}{\lambda}
  <  \log\left(\frac{1}{\delta}\right),
\]
and it follows that the average error probability is lower bounded by
\[
  \frac{1}{N} \sum_{j=1}^N \Pr\nolimits_j(\text{error}) >
  \exp\left( -c'' \cdot
    \frac{\eps^2}{\sigma^2} \cdot
  \frac{2n}{N}  \right)
  \ge
  \exp \left(- \log\left(\frac{1}{\delta}
  \right) \right) = \delta.
\]
Therefore, to attain an error probability
no higher than $\delta$, we must have
\[
  n = \Omega\left( \frac{\lambda \sigma}{
  \eps^2} \log\left(\frac{1}{\delta}\right)  \right)
\]
as desired.
\end{proof}

\section{Unknown Parameters}

\subsection{Proof of Theorem~\ref{thm: unknown eps} and Corollary~\ref{cor: unknown eps external} (Unknown Target Accuracy)}
\label{appendix: unknown target accuracy}

We first prove Theorem~\ref{thm: unknown eps}, the anytime-valid guarantee.
\begin{proof}[Proof of Theorem~\ref{thm: unknown eps}] 
    We first establish the high-probability correctness statement. Let
$\mathcal E_{\rm loc}$ be the event that the localization step succeeds.
By the same argument in Step 1 of Appendix~\ref{appendix: proof of main result}, we have
\[
    \Pr(\mathcal E_{\rm loc}^c)\le \delta_{\rm loc}.
\]
Conditioned on $\mathcal E_{\rm loc}$, a similar argument in Steps 2--6 of Appendix~\ref{appendix: proof of main result} gives the refinement guarantee
\[
    \Pr\left(
        |\hat\mu_\tau-\mu|>\eps_\tau
        \,\middle|\,
        \mathcal E_{\rm loc}
    \right)
    \le
    \delta_\tau
\]
for each round $\tau$. Thus, by the union bound and the summability of the confidence schedule, we have
\[
    \Pr\left(
        \mathcal E_{\rm loc}^c
        \, \cup \,
        \bigcup_{\tau\ge1}
        \left( \mathcal E_{\rm loc} \cap \{  |\hat\mu_\tau-\mu|>\eps_\tau\} \right)
    \right)
    \le
    \delta_{\rm loc}+\sum_{\tau\ge1}\delta_\tau
    \le
    \delta_{\rm loc}+\delta_{\rm ref}
    =
    \delta.
\]
On the complement of this event, every completed refinement estimate is
accurate. In particular, the final selected estimate satisfies
\[
    |\hat\mu_T-\mu|\le \eps_T
\]
for any stopping round $T$, including a transcript-dependent one.

It remains to compare $\eps_T$ (the accuracy of the last completed round) to $\eps^*$ (the oracle accuracy defined in~\eqref{eq: oracle_eps}).
Decompose the refinement sample complexity in~\eqref{eq: n_ref} into an accuracy-dependent scaling function $g_k(\sigma, \alpha)$ and a confidence term:
\begin{equation}\label{eq: n_ref_decompose}
    n_{\rm ref}(\alpha, \eta, k, \sigma) 
    = \Theta_k \left( g_k(\sigma, \alpha) \cdot \log\left(\frac{1}{\eta} \right) \right),
\end{equation}
where $g_k(\sigma, \alpha)$ is defined as:
\[
g_k(\sigma, \alpha) =
    \begin{cases}
       \left(  \dfrac{\sigma}{\alpha} \right)^2  & \text{if } k > 2
      \\ \\
       \left(  \dfrac{\sigma}{\alpha} \right)^2 \cdot \log \left( \dfrac{\sigma}{\alpha}\right)  & \text{if } k = 2
      \\ \\
       \left( \dfrac{\sigma}{\alpha}  \right)^{\frac{k}{k-1}}  & \text{if } k \in (1, 2).
    \end{cases}
\]
Let $p = 2$ for $k \ge 2$, and $p = \frac{k}{k-1} > 2$ for $k \in (1,2)$. Because the logarithmic penalty in the $k=2$ regime strictly increases as $\alpha$ shrinks, we can establish a geometric lower bound on the growth of $g_k(\sigma, \cdot)$. Specifically, we have
\begin{equation}\label{eq: g ratio bound}
    \frac{g_k(\sigma, \alpha_1)}{g_k(\sigma, \alpha_2)} = \Omega \left(  \left(\frac{\alpha_2}{\alpha_1} \right)^p \right)
    \quad \text{for any } 0 < \alpha_1 \le \alpha_2 \le \sigma.
\end{equation}
Fix a round $\tau$ and consider the prior rounds $s \le \tau$. Applying~\eqref{eq: g ratio bound} to $\alpha_1 = \eps_\tau$ and $\alpha_2 = \eps_s = 2^{\tau-s} \cdot \eps_\tau$ gives
\begin{equation}
\label{eq: g super geometric growth}
    g_k(\sigma, \eps_s) = O \left( 2^{-p(\tau-s)} \cdot g_k(\sigma, \eps_\tau) \right).
\end{equation}
Combining this geometric decay with the trivial bound $\log(1/\delta_s) \le \log(1/\delta_\tau)$ for all $s \le \tau$, the cumulative sample complexity is dominated by the final round:
\begin{align}
    \sum_{s=1}^{\tau} n_{\rm ref}(\eps_s, \delta_s, k, \sigma)
    &= O_k \left( \sum_{s=1}^{\tau} g_k(\sigma, \eps_s) \cdot \log(1/\delta_s) \right) 
    \notag \\
    &= O_k \left( g_k(\sigma, \eps_\tau) \cdot \log(1/\delta_\tau)  \cdot \sum_{s=1}^{\tau} 2^{-p(\tau-s)} \right)
    \notag \\
    &= O_k\left( g_k(\sigma, \eps_\tau) \cdot \log(1/\delta_\tau) \right).
    \label{eq: last term bounded}
\end{align}
We now relate the anytime performance to the oracle accuracy $\eps^*$, which satisfies
\begin{equation}
\label{eq: oracle implicit requirement}
    n_{\rm ref}(\eps^*, \delta_{\rm ref}, k, \sigma)
    = \Theta_k \Big( g_k(\sigma, \eps^*) \cdot \log(1/\delta_{\rm ref}) \Big) .
\end{equation}
Since round $T+1$ is not completed (see~\eqref{eq: last round T}), we have
\begin{equation}\label{eq: anytime budget upper}
    n_{\rm ref}(\eps^*, \delta_{\rm ref}, k, \sigma) < \sum_{s=1}^{T+1} n_{\rm ref}(\eps_s, \delta_s, k, \sigma).
\end{equation}
Substituting~\eqref{eq: oracle implicit requirement} into the left side of~\eqref{eq: anytime budget upper} and bounding the right side using~\eqref{eq: last term bounded} with $\tau = T+1$ yields
\[
    g_k(\sigma, \eps^*) \cdot \log(1/\delta_{\rm ref}) = 
    O_k \left( g_k(\sigma, \eps_{T+1}) \cdot \log(1/\delta_{T+1}) \right).
\]
Rearranging the terms to isolate the ratio of $g_k(\cdot)$ and substituting  $\delta_{T+1} = 6\delta_{\rm ref}/ (\pi^2 (T+1)^2)$ and $\delta_{\rm ref} = \delta/2$ gives
\begin{equation}
\label{eq: scaling bound}
    \frac{g_k(\sigma, \eps^*)}{g_k(\sigma, \eps_{T+1})}  
    = O_k \left(\frac{\log(1/\delta_{T+1})}{\log(1/\delta_{\rm ref})} \right)
    = O_k\left( 1 + \frac{\log (T+1)}{\log(1/\delta)} \right).
\end{equation}
To map this scaling bound back to the target accuracies, we consider two cases based on the relative size of $\eps^*$ and $\eps_{T+1}$:
\begin{itemize}[leftmargin=*]
    \item \textbf{Case 1} ($\eps^* \ge \eps_{T+1}$): 
    Because the target accuracy is halved at each round, we have $\eps_T = 2 \eps_{T+1} = O(\eps^*)$ trivially.
    
    \item \textbf{Case 2} ($\eps^* < \eps_{T+1}$): Applying~\eqref{eq: g ratio bound} to the left side of~\eqref{eq: scaling bound} with $\alpha_1 = \eps^*$ and $\alpha_2 = \eps_{T+1}$ yields
    \[
    \left( \frac{\eps_{T+1}}{\eps^*} \right)^p 
        = O \left( \frac{g_k(\sigma, \eps^*)}{g_k(\sigma, \eps_{T+1})} \right) = O_k \left( 1 + \frac{\log(T+1)}{\log(1/\delta)} \right),
    \]
   which implies
    \begin{equation}
    \label{eq:anytime_bound_with_logT}
        \eps_T = 2\eps_{T+1} 
        = O_k \left( \eps^* \left( 1 + \frac{\log(T+1)}{\log(1/\delta)} \right)^{\frac{1}{p}} \right).
    \end{equation}
    It remains to bound $T$. We first establish that $\eps_T \ge \eps^*$, which captures the intuition that ``anytime estimator cannot beat the oracle efficiency given the same budget''.
    Formally, this follows from the fact that round $T$ is completed and $n_{\rm ref}(\alpha, \eta, k, \sigma)$ is non-increasing in each of $\eta$ and $\alpha$:
    \[
    n_{\rm ref}(\eps^*, \delta_{\rm ref}, k, \sigma)
    \ge n_{\rm ref}(\eps_T, \delta_T, k, \sigma) 
    \ge n_{\rm ref}(\eps_T, \delta_{\rm ref}, k, \sigma)
    \implies 
    \eps_T \ge \eps^*
    \]
    By the confidence scheduling~\eqref{eq: parameter in round tau}, we have $2^T \le \sigma/\eps^*$ and thus $\log(T+1) = O(\log\log(e\sigma/\eps^*))$. Substituting this into~\eqref{eq:anytime_bound_with_logT} gives us the desired bound:
    \[
        \eps_T = O_k\left( \eps^* \left( 1 + \frac{\log\log (\sigma/\eps^*)}{\log(1/\delta)} \right)^{\frac{1}{p}} \right). 
    \]
\end{itemize}
    This completes the proof of Theorem~\ref{thm: unknown eps}.
\end{proof}

We now prove Corollary~\ref{cor: unknown eps external}, where the realized budget $n_{\rm true}$ is externally determined. The details are similar to those in proof of Theorem~\ref{thm: unknown eps} above, and we will omit details for brevity.
\begin{proof}[Proof of Corollary~\ref{cor: unknown eps external}]
    We first compare $\eps_T$ to $\eps^*$.
    Recall that we run the same halving schedule $\eps_\tau=\sigma/2^\tau$, but set $\delta_\tau=\delta_{\rm ref}$ for every $\tau\ge1$. In particular, $\delta_{T+1} = \delta_{\rm ref}$. A similar geometric domination argument from~\eqref{eq: n_ref_decompose}--\eqref{eq: scaling bound} yields
    \begin{equation*}
        \frac{g_k(\sigma, \eps^*)}{g_k(\sigma, \eps_{T+1})}  
        = O_k \left(\frac{\log(1/\delta_{T+1})}{\log(1/\delta_{\rm ref})} \right)
        = O_k(1).
    \end{equation*}
    A similar two-case analysis as in the proof of Theorem~\ref{thm: unknown eps} gives $\eps_T = O_k(\eps^*)$ for both cases.
    
    It remains to prove the high probability guarantee. Let $\mathcal E_{\rm loc}$ be the localization success event. As before, we have $\Pr(\mathcal E_{\rm loc}^c)\le\delta_{\rm loc}$.
   Condition on $\mathcal E_{\rm loc}$, the refinement guarantee gives 
   $\Pr\left(
        |\hat\mu_t-\mu|>\eps_t
        \,\middle|\,
        \mathcal E_{\rm loc}
    \right)
    \le
    \delta_{\rm ref}$
    for every fixed $t$.
    Because $n_{\rm true}$ is externally determined, the final round $T$ is independent of the samples, query responses, and internal randomness of the estimator. Hence conditioning on $T=t$ preserves the fixed-round guarantee:
    \[
    \Pr\left(|\hat\mu_T - \mu|>\eps_T\,\middle|\, \mathcal E_{\rm loc} \right)
    =
    \sum_{t\ge1}
    \Pr(T=t\mid \mathcal E_{\rm loc}) \cdot
    \Pr\left(|\hat\mu_t-\mu|>\eps_t
        \,\middle|\, T=t, \,  \mathcal E_{\rm loc}
    \right)
    \le
    \delta_{\rm ref}.
    \]
    Combining the above and applying the law of total probability gives
    \[
    \Pr\left(|\hat\mu_T-\mu|>\eps_T\right)
    \le
    \Pr(\mathcal E_{\rm loc}^c)
    +
    \Pr\left(
        |\hat\mu_T-\mu|>\eps_T
        \,\middle|\,
        \mathcal E_{\rm loc}
    \right)
    \le
    \delta_{\rm loc}+\delta_{\rm ref}
    =
    \delta.
    \]
    as desired.
\end{proof}

\subsection{Proof of Theorem~\ref{thm: partially unknown scale} (Adapting to Unknown Scale)}
\label{appendix: partially unknown scale}

Recall that in each round $i \in \{0, 1, \dots, T\}$, the algorithm invokes the main estimator with target accuracy $\eps_i = r \sigma_i / 6$, guessed scale parameter $\sigma_i$, and failure probability $\delta_i = \delta / (T+1)$.
We first bound the worst-case total sample complexity $n$, which occurs if the estimator does not halt early and run all $T+1$ loops.
Applying the upper bound from Theorem~\ref{thm: main} for the sample complexity $n\left(\eps_i, \delta_i, \lambda, \sigma_i \right)$ of each round $i$, and summing over all rounds yields
\begin{equation}
\label{eq: sum of rounds}
    n 
= \sum_{i=0}^T O_k\left( N_k\left(\frac{\eps_i}{\sigma_i} \right) \log\left(\frac{1}{\delta_i}\right) + \log\left(\frac{\lambda}{\sigma_i}\right) \right) = O\left( (T+1) \cdot N_k(r) \log\left(\frac{T+1}{\delta}\right) + \sum_{i=0}^T \log\left(\frac{\lambda}{\sigma_i}\right) \right),
\end{equation}
where 
\begin{equation*}
      N_k(r) = 
      \begin{cases} 
      \dfrac{1}{r^2} & \text{if } k > 2 
      \\ \\
      \dfrac{1}{r^2} \log\left(\dfrac{1}{r}\right) & \text{if } k = 2 
      \\ \\
      \left(\dfrac{1}{r}\right)^{\frac{k}{k-1}} & \text{if } k \in (1, 2)
      \end{cases}
  \end{equation*}
is the asymptotic scaling defined in Theorem~\ref{thm: partially unknown scale}) which satisfies $N_k(r/6) = \Theta_k(N_k(r))$. 
Recalling that $\sigma_i$ is a geometric sequence with $\sigma_0 = \sigma_{\max}$ and $\sigma_T  =  \Theta(\sigma_{\min})$ (see~\eqref{eq: T = log(sigma_max/sigma_min)} and~\eqref{eq: eps_i = r sigma_i/5}), we can evaluate the summation over the localization terms by
\[
\sum_{i=0}^T \log_2 \left(\frac{\lambda}{\sigma_i}\right) 
= \log_2 \left( \prod_{i=0}^{T} \frac{\lambda}{\sigma_i} \right)
= \log_2 \left(  \frac{\lambda^{T+1}}{  \left( \sqrt{\sigma_{0} \cdot \sigma_{T}} \right)^{T+1}  } \right)
= \Theta \left( T \log \frac{\lambda}{\sqrt{\sigma_{0} \cdot \sigma_{T}}} \right)
= \Theta \left( T \log \frac{\lambda}{\sqrt{\sigma_{\min} \cdot \sigma_{\max}}} \right),
\]
where the second equality follows from the fact that the product of a finite geometric sequence is its geometric mean raised to the number of terms (i.e., $\prod_{i=0}^T \sigma_i = (\sqrt{\sigma_0 \cdot \sigma_T})^{T+1}$).
Combining the above two findings and substituting $T = \left\lceil \log_2 \left(\sigma_{\text{max}} / \sigma_{\text{min}} \right) \right\rceil$ gives the desired sample complexity:
\begin{equation*}
      n = O\left( \log\left(\frac{\sigma_{\max}}{\sigma_{\min}}\right) \cdot \left( N_k(r) \cdot \log\left(\frac{\log(\sigma_{\max}/\sigma_{\min})}{\delta}\right) + \log\left(\frac{\lambda}{\sqrt{\sigma_{\min}\sigma_{\max}}}\right) \right) \right).
  \end{equation*}

We now show that selected output $\hat{\mu}^{(i^*)}$ is $(\eps, \delta)$-PAC for the relative target accuracy $\eps = r \sigma_{\mathrm{true}}$, i.e.,
\begin{equation}
\label{eq: eps-delta PAC with eps = r sigma}
    \Pr\left(  \left|\hat{\mu}^{(i^*)} - \mu \right| \le  r \sigma_{\mathrm{true}} \right) \ge 1- \delta.
\end{equation}
Let $j^*$ be the largest grid index (corresponding to the tightest valid scale) that still upper bounds the true scale, i.e.,
\begin{equation}
\label{eq: sigma_j* >= sigma_true}
    j^* = \max_{0 \le i \le T} \{ \sigma_i \ge \sigma_{\mathrm{true}} \}.
\end{equation}
Due to the geometric spacing $\sigma_i = \sigma_{\max} \cdot 2^{-i}$, we are guaranteed that the scale at $j^*$ tightly bounds the true scale:
\begin{equation}
\label{eq: sigma_j* < 2 sigma}
\sigma_{j^*} \le 2 \sigma_{\mathrm{true}} = \frac{2 \eps}{r}.
\end{equation}
For each round $i \le j^*$, the guessed parameter satisfies $\sigma_i \ge \sigma_{\mathrm{true}}$. Therefore, the distribution validly satisfies the assumed moment bound $\EE[|X-\mu|^k] \le \sigma_i^k$, ensuring that the subroutine's theoretical guarantees hold.  Let $\mathcal{E}_i = \{\mu \in I_i \}$ be the event that the true mean lies in the $i$-th confidence interval $I_i =  [ \hat{\mu}^{(i)} \pm \eps_i]$. By the subroutine's guarantee, the event  $\mathcal{E}_i$ occurs with probability at least $1 - \delta_i = 1 -\delta/(T+1)$. Applying the union bound over all rounds up to $j^*$, the ``good event'' $\mathcal{E} = \bigcap_{i \le j^*}
\mathcal{E}_i$ happens with probability at least 
\[
\Pr\left(  \mathcal{E} \right) =
\Pr\left(  \bigcap_{i\le j^\ast} \mathcal{E}_i \right) =
1 -  \Pr\left(  \bigcup_{i \le j^\ast} \neg
\mathcal{E}_i \right)
\ge 1 - \sum_{i \le j^\ast} \Pr\left( \neg
\mathcal{E}_i \right)
\ge 1- \sum_{i \le j^\ast} \delta_i
\ge 1- \sum_{i =0}^T \delta_i
= 1- \delta.
\]
We condition on the event $\mathcal{E}$ for the rest of the proof. 
Under event $\mathcal{E}$, we have $\mu \in I_i$ for all $i \le j^*$, and so all confidence intervals up to $j^*$ mutually intersect at the true mean $\mu$. Consequently, the algorithm's stopping condition ($I_i \cap I_l = \emptyset$ for some $l < i$) will not trigger at any step $i \le j^*$. Therefore, the last successful index $i^* = i-1$ must satisfy $i^* \ge j^*$, which implies
\begin{equation}
\label{eq: sigma_i <= sigma_j}
    \sigma_{i^*} \le \sigma_{j^*}
\end{equation}

To establish the PAC guarantee, we analyze the estimation error of $\hat{\mu}^{(i^*)}$. 
By the algorithm's acceptance criteria, because the interval $I_{i^*}$ was successfully accepted, it must intersect with all previously established intervals. In particular, because $j^* \le i^*$, there exists a common point $z \in I_{i^*} \cap I_{j^*}$.
By the definition of the intervals ((see~\eqref{eq: CI})), we have $|\hat{\mu}^{(i^*)} - z| \le \eps_{i^*}$ and $|\hat{\mu}^{(j^*)} - z| \le \eps_{j^*}$. 
Applying the triangle inequality yields
\[
|\hat{\mu}^{(i^*)} - \mu|
\le |\hat{\mu}^{(i^*)} - z| + |z - \hat{\mu}^{(j^*)}| + | \hat{\mu}^{(j^*)} - \mu| 
\le \eps_{i^*} + 2\eps_{j^*}.
\]
Using the target accuracy choice of $\eps_i = r \sigma_i/6$ and bounds~\eqref{eq: sigma_j* < 2 sigma}--\eqref{eq: sigma_i <= sigma_j}, we obtain the desired guarantee:
\begin{align*}
|\hat{\mu}^{(i^*)} - \mu|
\le \eps_{i^*} + 2\eps_{j^*} 
= \frac{r \sigma_{i^*}}{6}  + \frac{2r \sigma_{j^*}}{6} 
\le \frac{3r \sigma_{j^*}}{6} 
\le r \sigma_{\mathrm{true}} = \eps.
\end{align*}

\section{Details of Section~\ref{sec: two-stage} (Two-Stage Mean Estimator)}
\label{appendix: two-stage}

In this appendix, we provide the technical details for the non-adaptive
localization procedure described in Section~\ref{sec: two-stage}. The goal of
this procedure is to non-adaptively identify an interval $I$ of length
$O(\sigma)$ that contains the mean $\mu$ with high probability. The procedure
is fully non-adaptive and uses deterministic general 1-bit queries. At a high
level, it partitions the search range into $O(\lambda/\sigma)$ bins, assigns a
binary codeword of length
\[
    \ell = O\left(
        \log\frac{\lambda}{\sigma}
        +
        \log\frac{1}{\delta_{\rm loc}}
    \right)
\]
to each bin, and decodes by nearest neighbor in Hamming distance. The codeword
length $\ell$ is the sample complexity of the localization procedure.

We begin with a standard random-coding lemma. The proof follows from
Hoeffding's inequality and the probabilistic method.

\begin{lemma}[Balanced deterministic codebook]
\label{lem: balanced codebook}
For every integer $N\ge2$ and every integer $\ell \ge 10000\log N$, there exist
deterministic codewords $c_1,\dots,c_N\in\{0,1\}^{\ell}$ satisfying
\begin{equation}\label{eq:code_distance_bound}
    0.49 \ell \le d_H(c_i,c_j) \le 0.51 \ell
    \qquad \text{for all } i\neq j.
\end{equation}
\end{lemma}

\begin{proof}
Draw the codewords $c_1,\dots,c_N$ independently and uniformly at random
from $\{0,1\}^{\ell}$. For any fixed pair $i\neq j$, the Hamming distance
$d_H(c_i,c_j)$ has distribution $\operatorname{Binomial}(\ell,1/2)$. By
Hoeffding's inequality,
\[
    \Pr\left(
        \left|
        \frac{1}{\ell}d_H(c_i,c_j)-\frac{1}{2}
        \right| > 0.01
    \right)
    \le
    2\exp(-2\ell(0.01)^2)
    =
    2\exp(-0.0002 \ell).
\]
Taking a union bound over all $\binom{N}{2}$ pairs, the probability that any
pair violates the desired distance bound is at most
\[
    \binom{N}{2}\cdot 2\exp(-0.0002\ell)
    <
    N^2\exp(-0.0002 \ell)
    \le
    1,
\]
where the final inequality follows from $\ell \ge10000\log N$. Hence the desired
property holds with positive probability, so there exists a deterministic
codebook satisfying the stated distance bounds.
\end{proof}

\paragraph{The non-adaptive localization procedure.}
Given inputs $(\lambda,\sigma,\delta_{\rm loc})$, with
$\delta_{\rm loc}\in(0,1/2)$, the procedure proceeds as follows.

\begin{enumerate}
    \item \textbf{Initialization.} Set $h=20\sigma$. If $2\lambda\le h$,
    return the trivial interval $I=[-\lambda,\lambda]$.

    \item \textbf{Partitioning.} Otherwise, partition $[-\lambda,\lambda]$
    into $N=\left\lceil \frac{2\lambda}{h}\right\rceil$ equal-length bins
    $B_1,\dots,B_N$ of width $\Delta=\frac{2\lambda}{N} = \Theta(\sigma)$.
    Let $a_j=-\lambda+(j-1)\Delta$ for $j=1,\dots,N+1$. Define
    \[
        B_j=[a_j,a_{j+1}) \quad \text{for } j<N,
        \qquad
        B_N=[a_N,a_{N+1}].
    \]

    \item \textbf{Codebook selection.} Let
    \[
        \ell = \left\lceil 10000\left(
            \log N + \log \frac{1}{\delta_{\rm loc}}
        \right)
        \right\rceil.
    \]
    Fix a deterministic codebook $c_1,\dots,c_N\in\{0,1\}^{\ell}$
    satisfying $0.49\ell \le d_H(c_i,c_j) \le 0.51 \ell$ for all
    $i\neq j$. Such a codebook exists by Lemma~\ref{lem: balanced codebook}.

    \item \textbf{Querying.} For $x\in\mathbb R$, define the clipped bin
    index
    \[
        b(x)=
        \begin{cases}
            1, & x<-\lambda,\\
            j, & x\in B_j,\\
            N, & x>\lambda.
        \end{cases}
    \]
    For each sample $X_t$, $t=1,\dots, \ell$, the learner sends the
    deterministic non-adaptive query
    \[
        Q_t(x)=c_{b(x),t}
    \]
    and receives the bit
    \[
        Y_t=Q_t(X_t)=c_{b(X_t),t}.
    \]

    \item \textbf{Decoding.} For each candidate bin $j\in[N]$, compute the
    Hamming score
    \begin{equation}\label{eq:Hamming_score}
        H_j = \sum_{t=1}^{\ell} \mathbf 1\{Y_t\neq c_{j,t}\}.
    \end{equation}
    The decoder selects $\widehat i \in \arg\min_{j\in[N]} H_j$
    (i.e., using a nearest neighbor rule) and returns the enlarged interval
    \begin{equation}\label{eq:enlarged_interval}
        I = \bigcup_{u=\max\{1,\widehat i-2\}}^{\min\{N,\widehat i+2\}} B_u.
    \end{equation}
\end{enumerate}

\begin{proof}[Proof of Theorem~\ref{thm: alternative localization}]
We first consider the trivial case $2\lambda \le h$. In this case, the
procedure simply returns the entire search range $I = [-\lambda, \lambda]$.
Since $\mu \in [-\lambda, \lambda]$ by assumption and the returned interval
length is $|I| = 2\lambda \le h = 20\sigma = O(\sigma)$, the theorem's
guarantees hold deterministically.

Henceforth, we assume $2\lambda > h$. By the choices of
$N = \lceil 2\lambda / h \rceil$ and $\Delta = 2\lambda / N$, we have
\begin{equation}\label{eq:bin_width_bound}
    h/2 \le \Delta \le h.
\end{equation}
The returned interval consists of at most five contiguous bins, and so the
length of the returned interval is bounded by
\[
    |I| \le 5\Delta \le 5h = 100\sigma = O(\sigma).
\]
It remains to prove that $\mu \in I$ with probability at least
$1-\delta_{\rm loc}$.

Let $i=b(\mu)$ be the index of the bin containing the true mean, and define
the \emph{safe set} of bins as those within an index distance of at most two
from the true bin:
\[
    S = \{u \in [N] : |u - i| \le 2\}.
\]
If the bin $\widehat{i}$ selected by the decoder is in the safe set $S$, then
the enlarged interval in~\eqref{eq:enlarged_interval} contains the true bin
$B_i$, and hence contains $\mu$. Therefore, it suffices to prove
\[
    \Pr(\widehat i\in S)\ge1-\delta_{\rm loc}.
\]

We call any bin $j\notin S$ a \textit{far bin}. We claim that if
$b(x)\notin S$, then $|x-\mu|\ge h$. Indeed, if $x\in[-\lambda,\lambda]$ and
$b(x)\notin S$, then $x$ lies in a bin separated from $B_i$ by at least two
full bin widths, and $2\Delta\ge h$ by~\eqref{eq:bin_width_bound}. If $x$ is
clipped to a boundary bin outside $S$, the same lower bound holds because
moving outside $[-\lambda,\lambda]$ only increases the distance from the true
bin. Thus,
\[
    b(X) \notin S \implies |X - \mu| \ge h.
\]
Applying Markov's inequality and the assumption $\EE|X-\mu|\le\sigma$, we
obtain the following bound for the probability that a sample falls into a far
bin:
\begin{equation}\label{eq:heavy_bin_unlikely}
    \Pr(b(X) \notin S)
    \le \Pr(|X-\mu| \ge h)
    \le \frac{\EE|X-\mu|}{h}
    \le \frac{\sigma}{h}
    = \frac{1}{20},
\end{equation}
where the last step follows from our choice of $h=20\sigma$.

Denote
\[
    p_u \coloneqq \Pr(b(X)=u)
    \qquad \text{for each } u\in[N].
\]
Using this notation and bound~\eqref{eq:heavy_bin_unlikely}, the probability
that the bin index falls in the safe set $S$ is at least
\begin{equation}\label{eq:safe_bins_highly_likely}
    \sum_{u \in S} p_u
    =
    1-\sum_{u\notin S}p_u
    =
    1-\Pr(b(X)\notin S)
    \ge
    \frac{19}{20}.
\end{equation}
To prove that the nearest-neighbor decoder outputs a safe bin, it suffices to
exhibit one safe ``anchor'' bin whose Hamming score is smaller than the score
of every far bin. We take this anchor to be the heaviest safe bin:
\[
    i^* \in \arg\max_{u\in S} p_u.
\]
Since the safe set contains at most five bins ($|S|\le5$), the pigeonhole
principle gives
\begin{equation}\label{eq:anchor_bin_likely}
    p_{i^*}
    \ge
    \frac{1}{|S|}\sum_{u\in S} p_u
    \ge
    \frac{19}{100}.
\end{equation}

We next compare the Hamming scores of $i^*$ and a fixed far bin
$j\notin S$. For each candidate bin $u\in[N]$, recall that its Hamming score is
given by
\[
    H_u = \sum_{t=1}^{\ell} \mathbf 1\{Y_t\neq c_{u,t}\},
    \quad \text{where} \quad
    Y_t=c_{b(X_t),t}.
\]
Because the empirical scores tightly concentrate around their expected values,
we first express the expectation $\EE[H_u]$ in terms of the codebook distances.
Since the codebook is fixed and the samples are i.i.d., the random bin indices
$b(X_t)$ are i.i.d. with distribution $(p_1,\dots,p_N)$. By applying the law of
total probability over the bin indices, exchanging the order of summation, and
substituting the query function $Y_t=c_{b(X_t),t}$, we obtain
\begin{align}
    \EE[H_u]
    &= \sum_{t=1}^{\ell} \Pr(Y_t\neq c_{u,t})
    \notag\\
    &=
    \sum_{t=1}^{\ell}\sum_{v=1}^N
    \Pr(b(X_t)=v) \Pr(Y_t\neq c_{u,t}\mid b(X_t)=v)
    \notag\\
    &=
    \sum_{v=1}^N p_v
    \sum_{t=1}^{\ell} \mathbf 1\{c_{u,t}\neq c_{v,t}\}
    \notag\\
    &=
    \sum_{v=1}^N p_v d_H(c_u,c_v) 
    \notag\\
    &=
    \sum_{v\neq u} p_v d_H(c_u,c_v).
    \label{eq: expected_score_identity}
\end{align}

Applying~\eqref{eq: expected_score_identity} to the anchor bin $u=i^*$,
together with the upper code-distance bound~\eqref{eq:code_distance_bound}
and the lower bound~\eqref{eq:anchor_bin_likely} for $p_{i^*}$, yields
\[
    \EE[H_{i^*}]
    =
    \sum_{v\neq i^*}
    p_v d_H(c_{i^*},c_v)
    \le
    0.51\ell(1-p_{i^*})
    \le
    0.51\ell\left(1-\frac{19}{100}\right)
    =
    0.4131\ell.
\]
Conversely, for any far bin $j\notin S$, its codeword is distinct from all safe
bin codewords, i.e.,
\[
    d_H(c_j,c_u)\ge0.49\ell
    \qquad \text{for all } u\in S,
\]
by the lower code-distance bound in~\eqref{eq:code_distance_bound}. Restricting
the sum in~\eqref{eq: expected_score_identity} to the safe set and
applying~\eqref{eq:safe_bins_highly_likely}, we obtain
\[
    \EE[H_j]
    \ge
    \sum_{u\in S} p_u d_H(c_j,c_u)
    \ge
    0.49\ell \sum_{u\in S}p_u
    \ge
    0.49\ell\left(\frac{19}{20}\right)
    =
    0.4655\ell.
\]
Thus every far bin has expected Hamming score separated from the anchor bin by
a gap linear in the codeword length~$\ell$:
\begin{equation}\label{eq:anchor_gap}
    \EE[H_j-H_{i^*}]
    \ge
    0.4655\ell - 0.4131\ell
    =
    0.0524\ell
    >
    0.05\ell.
\end{equation}

We now show that this expected gap persists with high probability. For a fixed
far bin $j\notin S$, write
\[
    H_j-H_{i^*}
    =
    \sum_{t=1}^{\ell} A_t,
    \quad \text{where} \quad
    A_t
    =
    \mathbf 1\{Y_t\neq c_{j,t}\}
    -
    \mathbf 1\{Y_t\neq c_{i^*,t}\}.
\]
Because the samples $X_t$ are independent and the codebook is fixed, the
random variables $A_t$ are independent. Each satisfies $A_t\in[-1,1]$, meaning
their range is at most $2$. Applying Hoeffding's inequality and using the
expected gap from~\eqref{eq:anchor_gap}, we obtain
\[
    \Pr(H_j\le H_{i^*})
    \le
    \exp\left(-\frac{(0.05)^2\ell}{2}\right)
    =
    \exp(-0.00125\ell).
\]

Finally, because the decoder selects a bin of minimum Hamming score, if the
decoder selects a far bin, then some far bin must have score no larger than
the anchor's score:
\[
    \{\widehat i\notin S\}
    \subseteq
    \left\{
        \exists j\notin S:\ H_j\le H_{i^*}
    \right\}.
\]
Taking a union bound over all far bins, we have
\[
    \Pr(\widehat i\notin S)
    \le
    \sum_{j\notin S}\Pr(H_j\le H_{i^*})
    <
    N\exp(-0.00125\ell)
    \le
    N^{-11.5}\delta_{\rm loc}^{12.5}
    \le
    \delta_{\rm loc}.
\]
The last two inequalities follow from our chosen codeword length
$\ell \ge 10000(\log N+\log(1/\delta_{\rm loc}))$, as well as $N\ge2$ and
$\delta_{\rm loc}\in(0,1/2)$. Hence,
\[
    \Pr(\widehat i\in S)\ge1-\delta_{\rm loc}.
\]

Finally, the sample complexity of the localization procedure is exactly the
codeword length $\ell$, and
\[
    \ell
    =
    O\left(
        \log N+\log\frac{1}{\delta_{\rm loc}}
    \right)
    =
    O\left(
        \log\frac{\lambda}{\sigma}
        +
        \log\frac{1}{\delta_{\rm loc}}
    \right),
\]
since $N=\lceil 2\lambda/h\rceil$ and $h=20\sigma$.
\end{proof}

\section{Details of Section~\ref{sec:m_samples_per_query} (Multiple Samples per Query)} \label{app:m_samples_per_query}

We first establish the bound~\eqref{eq:sigma_m} for the scale parameter for the local average.

\begin{lemma}[Scale bound for the local average]\label{lem:scale_bound_local_avg}
  Let $Y_1,\dots,Y_m$ be i.i.d.\ copies of a zero-mean random variable
  $Y$ with $\EE|Y|^k \le \sigma^k$ for some $k>1$. Define the local
  average $\bar Y = \frac{1}{m} \sum_{j=1}^m Y_j$. Then
  \[
    \EE|\bar Y|^k \le \sigma_m^k,
  \]
  where
  \begin{equation}\label{eq:sigma_m_appendix}
    \sigma_m \coloneqq
    \begin{cases}
      C_k \sigma \cdot m^{-1/2},      & k \ge 2,\\
      C_k \sigma \cdot m^{-(1-1/k)},  & 1<k<2,
    \end{cases}
  \end{equation}
  and $C_k$ is a constant depending only on $k$.
\end{lemma}

\begin{proof}
  \textbf{Case $1<k<2$.} By the von~Bahr--Esseen inequality~\cite{von1965inequalities},
  \[
    \EE \left\lvert \sum_{j=1}^m Y_j \right\rvert^k
    \le 2\sum_{j=1}^m \EE|Y_j|^k
    \le 2m\sigma^k .
  \]
  Dividing by $m^k$ to transition to the local average $\bar{Y} = \frac{1}{m} \sum_{j=1}^m Y_j$ yields
  \[
    \EE|\bar Y|^k 
    = \frac{1}{m^k} \left( \EE \left| \sum_{j=1}^m Y_j \right|^k  \right)
    \le 2 \sigma^k \cdot m^{1-k}
    = \bigl(2^{1/k} \cdot \sigma \cdot m^{-(1-1/k)}\bigr)^k ,
  \]
  so the claim holds with $C_k = 2^{1/k}$.

  \textbf{Case $k\ge 2$.} We apply the Marcinkiewicz--Zygmund inequality in its squared-norm form (see, e.g., \cite[Exercise 6.2.6]{vershynin2026high}), which states that for independent zero-mean random variables $Y_1, \dots, Y_m$:
  \begin{equation}
    \left( \EE \left| \sum_{j=1}^m Y_j \right|^k \right)^{2/k} \leq C_k' \sum_{j=1}^m \left( \EE |Y_j|^k \right)^{2/k}
    \end{equation}
    for some constant $C_k'$.
    Given that the $Y_j$ are i.i.d. copies of $Y$ with $\EE |Y_j|^k  = \EE|Y|^k \leq \sigma^k$, this simplifies to
    \begin{equation}
    \left( \EE \left| \sum_{j=1}^m Y_j \right|^k \right)^{2/k} 
    \leq  C_k' \,  m  \sigma^2.
    \end{equation}
    Raising both sides to the power of $k/2$ yields
    \begin{equation}
        \EE \left| \sum_{j=1}^m Y_j \right|^k 
        \leq (C_k')^{k/2} \sigma^k \cdot m^{k/2}.
    \end{equation}
    Dividing by $m^k$ to transition to the local average $\bar{Y} = \frac{1}{m} \sum_{j=1}^m Y_j$ then yields
    \begin{equation}
        \EE |\bar{Y}|^k 
        = \frac{1}{m^k} \left( \EE \left| \sum_{j=1}^m Y_j \right|^k \right)
        \leq (C_k')^{k/2} \, \sigma^k m^{-k/2}
    \end{equation}
    This is equivalent to $\EE |\bar{Y}|^k \leq (C_k \sigma \cdot m^{-1/2})^k$ where $C_k = \sqrt{C_k'}$. Thus, we arrive at the form $\EE |\bar{Y}|^k \leq \sigma_m^k$ with $\sigma_m = C_k \sigma m^{-1/2}$.
\end{proof}

\begin{proof}[Proof of Corollary~\ref{cor:m_samples_per_query}]
    For each agent $t$, the local average $\widebar X_t$ has mean $\mu$, and applying Lemma~\ref{lem:scale_bound_local_avg} with $Y_j = X_{t,j}-\mu$, we obtain
\[
  \EE\bigl|\widebar X_t - \mu\bigr|^k \le \sigma_m^k,
\]
with $\sigma_m$ as in~\eqref{eq:sigma_m_appendix}. Thus the distribution of
$\widebar X_t$ belongs to the family $\mathcal{D}(k,\lambda,\sigma_m)$.
Since the corollary assumes $\sigma_m \ge \eps$, Theorem~\ref{thm: main}
applies directly with scale parameter $\sigma_m$. It yields an
$(\eps,\delta)$-PAC estimator with bit complexity
\[
  n_{\mathrm{bit}}
  = O\left( \log\left(\frac{\lambda}{\sigma_m}\right) \right)
    +
    \begin{cases}
      O_k\left( \dfrac{\sigma_m^2}{\eps^2} \cdot
                \log\left(\dfrac{1}{\delta}\right) \right) & k>2,\\ \\
      O\left(  \dfrac{\sigma_m^2}{\eps^2} \cdot
                \log\left(\dfrac{\sigma_m}{\eps}\right) \cdot
                \log\left(\dfrac{1}{\delta}\right) \right) & k=2,\\ \\
      O_k\left( \left(\dfrac{\sigma_m}{\eps}\right)^{\frac{k}{k-1}} \cdot
                \log\left(\dfrac{1}{\delta}\right) \right) & 1<k<2.
    \end{cases}
\]
The total number of raw samples drawn is $n = m \cdot n_{\mathrm{bit}}$.
Substituting the expression~\eqref{eq:sigma_m_appendix} for $\sigma_m$ and
absorbing all $k$-dependent constants into the $O_k$ notation gives the sample complexity in~\eqref{eq:sample_complexity_m_samples_per_query} of Corollary~\ref{cor:m_samples_per_query}.
\end{proof}

\end{document}